\documentclass[lettersize,journal]{IEEEtran}
\usepackage{amsmath,amsfonts,amssymb,amsthm}
\usepackage{algorithmic}
\usepackage{accents}  % for accentset
\usepackage{array}
\usepackage[caption=false,font=normalsize,labelfont=sf,textfont=sf]{subfig}
\usepackage{textcomp}
\usepackage{stfloats}
\usepackage{url}
\usepackage{verbatim}
\usepackage{graphicx}
\def\BibTeX{{\rm B\kern-.05em{\sc i\kern-.025em b}\kern-.08em
    T\kern-.1667em\lower.7ex\hbox{E}\kern-.125emX}}
\usepackage{balance}

\usepackage{xcolor}
\usepackage{hyperref}
\usepackage{dsfont}  % for \mathbbm{1}
\usepackage{multirow}
\usepackage{lipsum}
\newtheorem{definition}{Definition}
\newtheorem{theorem}{Theorem}
\newtheorem{corollary}{Corollary}[theorem]
\newtheorem{lemma}{Lemma}[section]

\newtheorem{remark}{Remark}
\newlength{\dhatheight}
\newcommand{\doublehat}[1]{%
    \settoheight{\dhatheight}{\ensuremath{\hat{#1}}}%
    \addtolength{\dhatheight}{-0.35ex}%
    \hat{\vphantom{\rule{1pt}{\dhatheight}}%
    \smash{\hat{#1}}}}

\DeclareMathOperator*{\argmin}{arg\,min}

\makeatletter
\def\blfootnote{\xdef\@thefnmark{}\@footnotetext}
\makeatother

\newif\iffirstrevproofread
\newcommand{\firstrev}[1]{%
\iffirstrevproofread
\textcolor{red}{#1}%  % note: first revision is done.
\else
#1%
\fi
}
\firstrevproofreadfalse

\newif\ifsecondrevproofread
\newcommand{\secondrev}[1]{%
\ifsecondrevproofread
\textcolor{red}{#1}%  % note: second revision is done.
\else
#1%
\fi
}
\secondrevproofreadfalse

\newif\ifthirdrevproofread
\newcommand{\thirdrev}[1]{%
\ifthirdrevproofread
\textcolor{red}{#1}%  % note: second revision is done.
\else
#1%
\fi
}
\thirdrevproofreadfalse

\begin{document}
\title{Parameterized Convex Universal Approximators \firstrev{for Decision-Making Problems}}
\author{
    Jinrae Kim
    and
    Youdan Kim, \firstrev{\textit{Senior Member, IEEE}}
    \thanks{Jinrae Kim is with the Department of Aerospace Engineering,
    Seoul National University, Seoul 08826, Republic of Korea \thirdrev{(e-mail: kjl950403@snu.ac.kr)}}
    \thanks{Youdan Kim is with the Department of Aerospace Engineering,
        Institute of Advanced Aerospace Technology,
    Seoul National University, Seoul 08826, Republic of Korea \thirdrev{(e-mail: ydkim@snu.ac.kr)}}
}

\markboth{Journal of \LaTeX\ Class Files,~Vol.~18, No.~9, September~2020}%
{How to Use the IEEEtran \LaTeX \ Templates}

\maketitle

\begin{abstract}
    Parameterized max-affine (PMA) and parameterized log-sum-exp (PLSE)
    networks are proposed for general decision-making problems.
    The proposed approximators generalize existing convex approximators, namely, max-affine (MA) and log-sum-exp (LSE) networks,
    by considering function arguments of condition and decision variables
    and replacing the network parameters of MA and LSE networks with continuous functions with respect to the condition variable.
    The universal approximation theorem of PMA and PLSE is proven,
    which implies that PMA and PLSE are \textit{shape-preserving universal} approximators for parameterized convex continuous functions.
    Practical guidelines for incorporating deep neural networks within PMA and PLSE networks are provided.
    A numerical simulation is performed to demonstrate the performance of the proposed approximators.
    The simulation results support that PLSE outperforms other existing approximators in terms of minimizer and optimal value errors with scalable \firstrev{and efficient computation} for high-dimensional cases.
\end{abstract}


\section{Introduction}
\label{sec:introduction}
\blfootnote{This paper has been accepted for publication by IEEE.
Copyright may be transferred without notice, after which this version may no longer be accessible.}
\blfootnote{© 2022 IEEE. Personal use of this material is permitted. Permission
from IEEE must be obtained for all other uses, in any current or future
media, including reprinting/republishing this material for advertising or
promotional purposes, creating new collective works, for resale or
redistribution to servers or lists, or reuse of any copyrighted
component of this work in other works.}
\blfootnote{Digital Object Identifier 10.1109/TNNLS.2022.3190198}
\IEEEPARstart{C}{}onditional decision-making problems find a suitable decision for a given condition; examples include
optimal control and reinforcement learning \cite{liberzonCalculusVariationsOptimal2012,suttonReinforcementLearningIntroduction2018}
and inference via optimization in energy-based learning
\cite{lecunTutorialEnergyBasedLearning2006}.
To make a decision in an optimal sense,
the decision-making problem can be solved by minimizing
a bivariate function for a given condition,
whose arguments consist of the condition and decision variables.
This formulation carries an implication: the bivariate function is regarded as a cost function,
which evaluates the \textit{cost} for a given condition and decision variables.
\thirdrev{
    Since most convex optimization problems are known to be relatively reliable and tractable for high dimensions,
    \firstrev{i.e., polynomial-time complexity} \cite{boydConvexOptimization2004,liuSurveyConvexOptimization2017,chunhuashenDualFormulationBoosting2010},
    incorporating convex optimization in decision-making
    is beneficial.
}
A natural approach is to implement a parameterized convex surrogate model for the cost function approximation,
where a bivariate function is said to be \textit{parameterized convex}
if the function is convex when the condition variables are fixed.
In this regard,
data-driven surrogate model approaches using (parameterized) convex approximators have drawn much attention
and have been implemented in many applications, including structured prediction and continuous-action Q-learning \cite{calafioreLogSumExpNeuralNetworks2020,calafioreEfficientModelFreeQFactor2020,amosInputConvexNeural2017a}.

Owing to the advances in machine learning,
performing data-driven function approximation has become essential and promising in many areas, including differential equations, generative model learning, and natural language processing with transformers
\cite{chenNeuralOrdinaryDifferential2018,goodfellowGenerativeAdversarialNetworks2014,vaswaniAttentionAllYou2017}.
One of the main theoretical results of function approximation is
referred to as the \textit{universal approximation theorem}.
In general,
if a function approximator is a universal approximator,
it is capable of approximating \textit{any function with arbitrary precision} (usually on a compact set).
For example,
it was shown that a shallow feedforward neural network (FNN) is a universal approximator for continuous functions \cite{cybenkoApproximationSuperpositionsSigmoidal1989,pinkusApproximationTheoryMLP1999a,hornikMultilayerFeedforwardNetworks1989,kolmogorovRepresentationContinuousFunctions1957},
and studies on the universal approximation theorem for deep FNN
have been conducted recently \cite{kidgerUniversalApproximationDeep2020}.
On the other hand,
one may want to approximate a class of functions specified by a certain \textit{shape}, for example, monotonicity and convexity, while the approximator preserves the same shape.
This is called the \textit{shape-preserving approximation} (SPA) \cite{devoreConstructiveApproximation1993}.
\thirdrev{
    For example,
    max-affine (MA) \cite{ghoshMaxAffineRegressionParameter2022} and log-sum-exp (LSE) networks
    were proposed as shape-preserving universal approximators for convex continuous functions \cite{calafioreLogSumExpNeuralNetworks2020}.
}
\thirdrev{
    Also, difference of LSE (DLSE) network was proposed as an ordinary universal approximator
    for continuous functions.
    Note that the DLSE network can approximately infer the optimizer
    fast by utilizing difference of convex optimization
    \cite{calafioreUniversalApproximationResult2020}.
}
For decision-making problems, however,
one may need a shape-preserving approximator for
parameterized convex functions, not just convex functions.
An attempt was conducted to suggest a parameterized convex approximator,
the partially input convex neural network (PICNN).
However, it has not been shown that PICNN is a universal approximator \cite{amosInputConvexNeural2017a}.

To address this issue,
in this study,
new parameterized convex approximators are proposed:
parameterized MA (PMA) and parameterized LSE (PLSE) networks.
PMA and PLSE networks are extensions of MA and LSE networks to parameterized convex functions by replacing the parameters of MA and LSE networks with continuous functions with respect to condition variables.
This study demonstrates that PMA and PLSE are \textit{shape-preserving universal} approximators for parameterized convex continuous functions.
The main challenge of showing that PMA and PLSE are universal approximators comes from the replacement of network parameters with continuous functions.
In the construction of MA and LSE networks,
subgradients are arbitrarily selected from corresponding subdifferentials \cite{calafioreLogSumExpNeuralNetworks2020},
while \firstrev{the subgradients are replaced with} functions of the condition variables in PMA and PLSE networks.
Therefore,
the subgradient functions should carefully be approximated along with condition variable axes.
This issue is resolved by \textit{continuous selection} of multivalued subdifferential mappings.
From the practical point of view,
guidelines for practical implementation are provided, for example, by utilizing the deep architecture.
Numerical simulation is performed to demonstrate the proposed approximators' approximation capability in terms of minimizer and optimal value errors as well as \firstrev{the solving time of convex optimization for decision-making}.
The simulation result highly supports that the proposed approximators, particularly the PLSE network, show the smallest minimizer and optimal value errors from low- to high-dimensional cases as well as scalable \firstrev{solving time} compared to existing approximators, FNN, MA, LSE, and PICNN.

The rest of this paper is organized as follows.
In \autoref{sec:preliminaries},
some mathematical backgrounds for set-valued analysis and convex analysis are briefly summarized,
and several types of universal approximators are defined.
In \autoref{sec:main_result},
the proposed parameterized convex approximators,
PMA and PLSE networks, are introduced.
Additionally, \autoref{sec:main_result} describes
the main theoretical results of this study.
The main results include
that PMA and PLSE are parameterized convex universal approximators
and can collaborate with ordinary universal approximators,
for example, the multilayer FNN, to approximate continuous selection of subdifferential mappings.
In \autoref{sec:numerical_simulation},
numerical simulation is performed to demonstrate the approximation capability of the proposed approximators,
in terms of minimizer and optimal value errors as well as \firstrev{solving time}.
Finally,
\autoref{sec:conclusion} concludes this study with a summary and 
future works.

\section{Preliminaries}
\label{sec:preliminaries}
Let $\mathbb{N}$ and $\mathbb{R}$ be the sets of all natural and real numbers, respectively.
The extended real number line is denoted by $\overline{\mathbb{R}} := \mathbb{R} \cup \{-\infty, +\infty\}$.
In this study, it is assumed that functions are defined in Euclidean vector space with a standard inner product
$\langle x, y \rangle := x^{\intercal} y$ and Euclidean norm $\lVert x \rVert := \sqrt{\langle x, x \rangle}$.
Given a set $U$, the interior and closure of $U$ are denoted by $\accentset{\circ}{U}$ and $\overline{U}$, respectively.
The diameter of a set $U$ is defined as $\text{diam}(U) := \sup_{x, y \in U} \lVert x - y \rVert$.
The supremum norm of a function $f$ is defined as
$\lVert f \rVert_{\infty} := \sup_{x \in X} \lVert f(x) \rVert$
where $X$ is the domain of the function $f$,
or an appropriate set in context.
$L$-Lipschitzness stands for the Lipschitzness
with a specific Lipschitz constant of $L$ and is defined as follows.
\begin{definition}[$L$-Lipschitzness]
    A real-valued function $f: X \to \mathbb{R}$ is said to be
    \textup{$L$-Lipschitz} for some $L > 0$ (simply \textup{Lipschitz} or \textup{Lipschitz continuous})
    if for all $x, y \in X$ it satisfies $\lvert f(x) - f(y) \rvert \leq L \lVert x -y  \rVert$.
\end{definition}

\subsection{Set-valued analysis}
If a function $f$ is defined on $X$, whose value is a subset of $Y$, i.e.,
$f(x) \subset Y, \forall x \in X$, then $f$ is called a \textit{multivalued function} (or a \textit{set-valued function})
and denoted as $f:X \to Y$. Ordinary functions can be regarded as a single-valued function in context, i.e.,
$f(x) = \{ y \}$ for a certain $y \in Y$.

A \textit{graph} of a multivalued function $f: X \to Y$ is defined by
\begin{equation}
    \text{Graph}(f) := \{ (x, y) \in X \times Y \vert y \in f(x) \}.
\end{equation}

A multivalued function $f:X \to Y$ is said to be \textit{upper hemicontinuous (u.h.c.)} at $x_{0}$
if, for any open neighborhood $V$ of $f(x_{0})$, there exists a neighborhood $U$ of $x_{0}$ such that
for all $x \in U$, $f(x)$ is a subset of $V$.

\firstrev{
    A ball of radius $r > 0$ around $K$ in $X$ is denoted by
    $B_{X}(K, r) := \{ x \in X \vert \inf_{y \in K} \lVert x-y \rVert \leq r \}$.
    If there is no confusion, let $B(K, r) := B_{X}(K, r)$.
    The \textit{unit ball} is denoted by $B_{X}$,
    and therefore $B_{X}(K, r) = \overline{K + r B_{X}}$.
}
$\epsilon$-selection,
a notion of the approximation of a multivalued function
with a specific accuracy by a single-valued function,
can be defined as follows.
\begin{definition}[$\epsilon$-selection]
    Given multivalued functions $f:X \to Y$ and $\epsilon > 0$, if there exists a single-valued function
    $g:X \to Y$ such that $\text{Graph}(g) \subset B(\text{Graph}(f), \epsilon)$,
    $g$ is said to be an \textup{$\epsilon$-selection} of $f$.
\end{definition}
It is said to be a continuous selection if the selection is continuous.

\subsection{Convex analysis}
A function $f: X \to \overline{\mathbb{R}}$ is said to be \textit{lower semicontinuous (l.s.c.)} at $x_{0}$
if for every $\epsilon > 0$, there exists a neighborhood $U$ of $x_{0}$ such that $f(x) \geq f(x_{0}) - \epsilon$ for all $x \in U$
where $f(x_{0}) < \infty$, and $f(x)$ tends to $+\infty$ as $x \to x_{0}$ when $f(x_{0}) = +\infty$.
In other words, $\liminf_{x \to x_{0}} f(x) \geq f(x_{0})$.

\thirdrev{
    Let $f: X \to \overline{\mathbb{R}}$ be a convex function \cite{boydConvexOptimization2004}.
}
A convex function $f$ is called \textit{proper}
if $f(x) > -\infty, \forall x \in \text{dom}f \neq \emptyset$
where the effective domain of $f$ is defined as $\text{dom}f := \{ x \in X \vert f(x) < +\infty \}$.

Given function $f: X \to \overline{\mathbb{R}}$, convex conjugate (a.k.a. Legendre-Fenchel transformation)
of $f$ is the function $f^{*}: X^{*} \to \overline{\mathbb{R}}$, where the value at $x^{*} \in X^{*}$ is
\begin{equation}
    f^{*}(x^{*}) = \sup_{x \in X} \{ \langle x, x^{*} \rangle - f(x)\},
\end{equation}
where $X^{*}$ is the dual space of $X$.

If $f:X \to \mathbb{R}$ is a convex function defined on a convex open set in $\mathbb{R}^{n}$,
a vector $v$ is called a \textit{subgradient} at $x_{0} \in X$ if for any $x \in X$ one has that

\begin{equation}
    f(x) \geq f(x_{0}) + \langle v, x-x_{0} \rangle.
\end{equation}
The set of all subgradients at $x_{0}$ is called \textit{subdifferential} at $x_{0}$, denoted by $\partial f(x_{0})$.
That is, $\partial f(x) := \{ x^{*} \in \mathbb{R}^{n} \vert f(\tilde{x}) \geq f(x) + \langle x^{*}, \tilde{x} - x \rangle, \forall \tilde{x} \in X \}$.

Now,
we define a class of functions.
\begin{definition}[Parameterized convexity]
    A function $f: X \times U \to \mathbb{R}$ is said to be \textup{parameterized convex}
    (with respect to the second argument)
    if for any $x \in X$, $f(x, \cdot)$ is convex.
\end{definition}

\subsection{Related works and existing universal approximators}
A universal approximation theorem (UAT) \cite{cybenkoApproximationSuperpositionsSigmoidal1989,pinkusApproximationTheoryMLP1999a} describes
a kind of approximation capability of an approximator.
UATs usually consider continuous functions on a compact subspace
of Euclidean vector space.
In this study, universal approximators for continuous functions
are referred to as ordinary universal approximators,
which are defined as follows.
\begin{definition}[Ordinary universal approximator]
    Given a compact subspace $X$ of $\mathbb{R}^{n}$,
    let $C(X, Y)$ be the collection of all continuous functions from $X$ to $Y$,
    and $C(X) := C(X, Y)$ if there is no confusion.
    A collection $\mathcal{F}$ of continuous functions defined on $X$ is said to be
    an \textup{(ordinary) universal approximator} if $\mathcal{F}$
    is dense in $C(X)$, i.e.,
    $\forall f \in C(X), \forall \epsilon > 0, \exists \hat{f} \in \mathcal{F}$ such that
    $\lVert \hat{f} - f \rVert_{\infty} < \epsilon$.
\end{definition}
Examples of ordinary universal approximators include
a single-hidden-layer FNN
\cite{cybenkoApproximationSuperpositionsSigmoidal1989,pinkusApproximationTheoryMLP1999a}.
Another class of universal approximators for convex continuous functions, which preserves the convexity,
is defined as follows.
\begin{definition}[Convex universal approximator]
    Given a compact convex subspace $X$ of $\mathbb{R}^{n}$,
    let $C^{\text{conv}}(X, Y)$ be the collection of all convex continuous functions from $X$ to $Y$,
    and $C^{\text{conv}}(X) := C^{\text{conv}}(X, Y)$ if there is no confusion.
    A collection $\mathcal{F}$ of convex continuous functions defined on $X$ is said to be
    a \textup{convex universal approximator} if $\mathcal{F}$
    is dense in $C^{\text{conv}}(X)$, i.e.,
    $\forall f \in C^{\text{conv}}(X), \forall \epsilon > 0, \exists \hat{f} \in \mathcal{F}$ such that
    \thirdrev{
        $\lVert \hat{f} - f \rVert_{\infty} < \epsilon$.
    }
\end{definition}
Examples of ordinary universal approximators include max-affine (MA) and log-sum-exp (LSE) networks \cite{calafioreLogSumExpNeuralNetworks2020}.
The MA network is constructed as a pointwise supremum of supporting hyperplanes of a given convex function,
which are underestimators of the convex function.
The LSE network is a smooth version of the MA network that replaces the pointwise supremum
with a log-sum-exp operator.
MA and LSE networks can be represented with some $I \in \mathbb{N}$,
$a_{i} \in \mathbb{R}^{m}$, $b_{i} \in \mathbb{R}$ for $i = 1, \ldots, I$, and $ T > 0$ as
\begin{equation}
    \label{eq:ma_and_lse}
    \begin{split}
        f^{\text{MA}}(u)
        &= \max_{1 \leq i \leq I} \left( \langle a_{i}, u \rangle + b_{i} \right),
        \\
        f^{\text{LSE}}(u)
        &= T \log \left( \sum_{i=1}^{I} \exp \left( \frac{\langle a_{i}, u \rangle + b_{i}}{T} \right) \right).
    \end{split}
\end{equation}

Hereafter, we adopt the following notation for brevity:
Condition (state) and decision (action or input) variables are denoted by $x \in X$ and $u \in U$, respectively,
where $X$ and $U$ denote condition and decision spaces, respectively.
It is assumed that $X$ is a compact subspace of $\mathbb{R}^{n}$ and that $U$ is a convex compact subspace of $\mathbb{R}^{m}$.

\section{Main Result}
\label{sec:main_result}
In this section,
two parameterized convex approximators are proposed,
the parameterized max-affine (PMA) and parameterized log-sum-exp (PLSE) networks,
and the main results of this study are presented.
The main results \firstrev{are:}
i) PMA and PLSE networks are parameterized convex universal approximators,
ii) the continuous functions $a_{i}$ and $b_{i}$ in Eq. \eqref{eq:pma_and_plse} can be replaced by ordinary universal approximators for practice implementation,
\firstrev{and} iii) under a mild assumption, the results also hold for conditional decision space settings, that is,
PMA and PLSE are parameterized convex universal approximators
even when a conditional decision space mapping
$\secondrev{U_{cond}}: X \to \mathbb{R}^{m}$ is given,
where the decision $u$ must be in $\secondrev{U_{cond}}(x)$ for a given condition $x$.

\subsection{Proposed parameterized convex approximators}
Universal approximators for parameterized convex continuous functions are defined as follows.
\begin{definition}[Parameterized convex universal approximator]
    Given a compact subspace $X$ of $\mathbb{R}^{n}$ and a compact convex subspace $U$ of $\mathbb{R}^{m}$,
    let $C^{\text{p-conv}}(X \times U, Y)$ be the collection of all parameterized convex continuous functions from $X \times U$ to $Y$,
    and $C^{\text{p-conv}}(X \times U) := C^{\text{p-conv}}(X \times U, Y)$ if there is no confusion.
    A collection $\mathcal{F}$ of parameterized convex continuous functions defined on $X \times U$ is said to be
    a \textup{parameterized convex universal approximator} if $\mathcal{F}$
is dense in $C^{\text{p-conv}}(X \times U)$, i.e.,
    $\forall f \in C^{\text{p-conv}}(X \times U), \forall \epsilon > 0, \exists \hat{f} \in \mathcal{F}$ such that
    $\lVert \hat{f} - f \rVert_{\infty} < \epsilon$.
\end{definition}
Let us introduce the PMA and PLSE networks,
which are the generalized MA and LSE networks for parameterized convex function approximation.
For some $T > 0$,
let $\mathcal{F}^{\text{PMA}}$ and $\mathcal{F}_{T}^{\text{PLSE}}$ be the collection of all PMA and PLSE networks, respectively,
where each PMA and PLSE network can be represented with some $I \in \mathbb{N}$, $a_{i} \in \mathcal{C}(X, \mathbb{R}^{m})$, $b_{i} \in \mathcal{C}(X, \mathbb{R})$, for $i = 1, \ldots, I$, and $T > 0$ as
\begin{equation}
    \label{eq:pma_and_plse}
    \begin{split}
        f^{\text{PMA}}(x, u)
        &= \max_{1 \leq i \leq I} \left( \langle a_{i}(x), u \rangle + b_{i}(x) \right),
        \\
        f^{\text{PLSE}}(x, u)
        &= T \log \left( \sum_{i=1}^{I} \exp \left( \frac{\langle a_{i}(x), u \rangle + b_{i}(x)}{T} \right) \right),
    \end{split}
\end{equation}
where $T > 0$ is usually referred to as the \textit{temperature}.
Compared to MA and LSE in Eq. \eqref{eq:ma_and_lse},
PMA and PLSE generalize MA and LSE to be parameterized convex
by replacing network parameters $a_{i} \in \mathbb{R}^{m}$ and $b_{i} \in \mathbb{R}$ with continuous functions $a_{i} \in \mathcal{C}(X, \mathbb{R}^{m})$ and $b_{i} \in \mathcal{C}(X, \mathbb{R})$ for $i = 1, \ldots, I$.
Note from Eq. \eqref{eq:pma_and_plse} that PMA and PLSE networks are indeed parameterized convex.

In the following \autoref{thm:diff_PMA_and_PLSE},
it is shown that a PLSE network can be made arbitrarily close to
the corresponding PMA network
with a sufficiently small temperature.
\begin{theorem}
    \label{thm:diff_PMA_and_PLSE}
    Given $T > 0$, $I \in \mathbb{N}$,
    \thirdrev{
        $a_{i} \in \mathcal{C}(X, \mathbb{R}^{m})$,
    }
    and $ b_{i} \in \mathcal{C}(X, \mathbb{R})$
    for $i=1, \ldots, I$,
    let
    $f^{\text{PMA}}$ and $f^{\text{PLSE}}$ be the PMA and PLSE networks constructed as in Eq. \eqref{eq:pma_and_plse}, respectively.
    Then, for all $(x, u) \in X \times U$,
    the following inequalities hold,
    \begin{equation}
        f^{\text{PMA}}(x, u)
        \leq f^{\text{PLSE}}(x, u)
        \leq T \log{I} + f^{\text{PMA}}(x, u).
    \end{equation}
\end{theorem}
\begin{proof}
    The proof is merely an extension of \cite[Lemma 2]{calafioreLogSumExpNeuralNetworks2020} to the case of parameterized convex approximators.
    For completeness, the proof is \firstrev{shown} here.

    It can be deduced from Eq. \eqref{eq:pma_and_plse} that
    \begin{equation}
        \begin{split}
            &f^{\text{PMA}}(x, u)
            \\
            &= \max_{1 \leq i \leq I} \left( \langle a_{i}(x), u \rangle + b_{i}(x) \right)
            \\
            &= \max_{1 \leq i \leq I} T \log \left( \left( \exp \left(
                    \langle a_{i}(x), u \rangle + b_{i}(x)
            \right) \right)^{1/T} \right)
            \\
            &= T \log \left( \max_{1 \leq i \leq I}  \exp \left(
                    \frac{\langle a_{i}(x), u \rangle + b_{i}(x)}{T}
            \right)  \right)
            \\
            &\leq T \log \left( \sum_{i=1}^{I}  \exp \left(
                    \frac{\langle a_{i}(x), u \rangle + b_{i}(x)}{T}
            \right)  \right)
            \\
            &= f^{\text{PLSE}}(x, u)
            ,
        \end{split}
    \end{equation}
    which proves the first inequality.
    The second inequality can be derived as follows.
    \begin{equation}
        \begin{split}
            &f^{\text{PLSE}}(x, u)
            \\
            &= T \log \left( \sum_{i=1}^{I}  \exp \left(
                    \frac{\langle a_{i}(x), u \rangle + b_{i}(x)}{T}
            \right)  \right)
            \\
            &\leq T \log \left(I \exp \left( \max_{1 \leq i \leq I} 
                    \frac{\langle a_{i}(x), u \rangle + b_{i}(x)}{T}
            \right)  \right)
            \\
            &= T \log \left(I \exp \left( \left(
                    f^{\text{PMA}}(x, u)
            \right) \right)^{1/T} \right)
            \\
            &= T \log I + f^{\text{PMA}}(x, u)
            ,
        \end{split}
    \end{equation}
    which concludes the proof.
\end{proof}
\autoref{thm:diff_PMA_and_PLSE} implies that for any $\epsilon > 0$,
$\lVert f^{\text{PLSE}}_{T} - f^{\text{PMA}} \rVert_{\infty} < \epsilon$ for all $T \in \left(0, \frac{\epsilon}{\log{I}} \right)$.

\subsection{$\epsilon$-selection of subdifferential mapping}
To prove that MA and LSE are convex universal approximators,
as in \cite{calafioreLogSumExpNeuralNetworks2020},
a dense sequence of points in the decision space is selected,
and corresponding subgradient vectors are selected arbitrarily
from subdifferentials at each point of the dense sequence.
However,
to extend MA and LSE to PMA and PLSE, that is,
to prove that PMA and PLSE are parameterized convex universal approximators,
the main difficulty arises from
the fact that the subgradient vectors that appeared in MA and LSE become
functions of the condition variable in PMA and PLSE,
and therefore the subgradient vectors cannot be selected arbitrarily.
The following theorem addresses how to deal with this issue:
each subdifferential mapping, a function of the condition variable,
can be approximated by a continuous selection
of the corresponding multivalued function.
\begin{theorem}
    \label{thm:eps_selection_of_subdiff_mapping}
    Let $f: X \times U \to \mathbb{R}$ be a parameterized convex continuous function.
    Suppose for all $x \in X$ that $f_{x}(u) := f(x, u)$ is $L$-Lipschitz.
    Given $u \in U$, let $\Gamma_{u} : X \to \mathbb{R}^{m}$ be a multivalued function
    such that $\Gamma_{u}(x) := \partial f_{x}(u)$ for all $x \in X$.
    Suppose that $U$ has a nonempty interior.
    Given a sequence $\left\{ u_{i} \in \accentset{\circ}{U} \right\}_{i \in \mathbb{N}}$,
    for all $\epsilon > 0$, there exist $\epsilon$-selections $\hat{u}_{\epsilon, i}^{*}: X \to \mathbb{R}^{m}$
    of $\Gamma_{u_{i}}(x)$ for all $i \in \mathbb{N}$.
    Additionally, a sequence of $\epsilon$-selections, $\{ \hat{u}_{\epsilon, i}^{*} \}_{i \in \mathbb{N}}$, is equi-Lipschitz.
\end{theorem}
\begin{proof}
    See
    % \autoref{sec:proof_of_eps_selection_of_subdiff_mapping}.
    \hyperref[sec:proof_of_eps_selection_of_subdiff_mapping]{Appendix A}.
\end{proof}

\subsection{Universal approximation theorem}
The main UAT results are provided in the following;
PMA and PLSE networks can be arbitrarily close to any parameterized convex continuous functions on the product of condition and decision spaces.
\begin{theorem}[PMA is a parameterized convex universal approximator]
    \label{thm:pma_is_a_parameterized_convex_universal_approximator}
    Given a parameterized convex continuous function $f: X \times U \to \mathbb{R}$,
    for any $\epsilon > 0$, there exists a PMA network $\hat{f} \in \mathcal{F}^{PMA}$
    such that $\lVert \hat{f} - f \rVert_{\infty} < \epsilon$.
\end{theorem}
\begin{proof}
    See
    % \autoref{sec:proof_of_pma_is_a_parameterized_convex_universal_approximator}.
    \hyperref[sec:proof_of_pma_is_a_parameterized_convex_universal_approximator]{Appendix B}.
\end{proof}

\begin{corollary}[PLSE is a parameterized convex universal approximator]
    \label{cor:plse_is_a_parameterized_convex_universal_approximator}
    Given a parameterized convex continuous function $f: X \times U \to \mathbb{R}$,
    for any $\epsilon > 0$, there exists a positive constant $\overline{T} > 0$ such that
    for all $T \in (0, \overline{T})$, there exists a PLSE network $\hat{f} \in \mathcal{F}_{T}^{PLSE}$
    such that $\lVert \hat{f} - f \rVert_{\infty} < \epsilon$.
\end{corollary}
\begin{proof}
    By \autoref{thm:pma_is_a_parameterized_convex_universal_approximator},
    given $\epsilon > 0$,
    there exists a PMA network $\hat{f}^{\text{PMA}} \in \mathcal{F}^{\text{PMA}}$
    (with $I \in \mathbb{N}$ in Eq. \eqref{eq:pma_and_plse})
    such that
    $\lVert \hat{f}^{\text{PMA}} - f \rVert_{\infty} < \epsilon / 2$.
    By \autoref{thm:diff_PMA_and_PLSE}, setting $\overline{T} = \frac{\epsilon}{2 \log I}$ and letting $\hat{f}^{\text{PLSE}} \in \mathcal{F}^{\text{PLSE}}_{T}$ be the corresponding PLSE network imply that
    \begin{equation}
        \begin{split}
            & \lVert \hat{f}^{\text{PLSE}} - f \rVert_{\infty}
            \leq \lVert \hat{f}^{\text{PLSE}} - \hat{f}^{\text{PMA}} \rVert_{\infty} + \lVert \hat{f}^{\text{PMA}} - f \rVert_{\infty}
            \\
            & < \frac{\epsilon}{2} + \frac{\epsilon}{2} = \epsilon
            ,
        \end{split}
    \end{equation}
    for all $T \in (0, \overline{T})$, which concludes the proof.
\end{proof}

\subsection{Implementation guidelines}
Although it is shown from \autoref{thm:pma_is_a_parameterized_convex_universal_approximator}
and \autoref{cor:plse_is_a_parameterized_convex_universal_approximator}
that PMA and PLSE networks have enough capability to approximate
any parameterized convex continuous functions,
it is hard in practice to directly find the continuous functions $a_{i}$'s and $b_{i}$'s that appear in Eq. \eqref{eq:pma_and_plse}.
For practical implementations,
one would utilize ordinary universal approximators to approximate $a_{i}$'s and $b_{i}$'s.
The following theorem supports that PMA and PLSE can be constructed
with ordinary universal approximators to make them practically implementable while not losing their approximation capability.
\begin{theorem}[PMA with ordinary universal approximators is a parameterized convex universal approximator]
    \label{thm:pma_with_ordinary_universal_approximator}
    Let $\mathcal{F}_{1}$ and $\mathcal{F}_{2}$ be ordinary universal approximators on $X$ to $\mathbb{R}^{m}$ and
    $\mathbb{R}$, respectively.
    Given parameterized convex continuous function $f: X \times U \to \mathbb{R}$,
    for any $\epsilon > 0$, there exist $I \in \mathbb{N}$, $\hat{a}_{i} \in \mathcal{F}_{1}$, and $\hat{b}_{i} \in \mathcal{F}_{2}$
    for $i = 1, \ldots, I$ such that
    $\lVert \doublehat{f} -f  \rVert_{\infty} < \epsilon$ where
    $\doublehat{f}(x, u) := \max_{1 \leq i \leq I} \{ \langle \hat{a}_{i}(x), u \rangle + \hat{b}_{i} \} \in \mathcal{F}^{PMA}$.
\end{theorem}
\begin{proof}
    By \autoref{thm:pma_is_a_parameterized_convex_universal_approximator},
    $\forall \epsilon > 0$, $\exists \hat{f} \in \mathcal{F}^{\text{PMA}}$
    such that $\lVert \hat{f} - f \rVert_{\infty} < \epsilon / 2$
    where the PMA network $\hat{f}$ is defined as in Eq. \eqref{eq:pma_and_plse}.
    Additionally,
    since $\mathcal{F}_{1}$ and $\mathcal{F}_{2}$ are ordinary universal approximators on $X$ to $\mathbb{R}^{m}$ and $\mathbb{R}$, respectively,
    given $\epsilon > 0$, $a_{i} \in \mathcal{C}(X, \mathbb{R}^{m})$, $b_{i} \in \mathcal{C}(X, \mathbb{R})$ for $i = 1, \ldots, I$,
    there exist $\hat{a}_{i} \in \mathcal{F}_{1}$, $\hat{b}_{i} \in \mathcal{F}_{2}$ such that
    $\lVert \hat{a}_{i} - a_{i} \rVert_{\infty} < \frac{\epsilon}{4 \text{diam}(U)}$ and $\lVert \hat{b}_{i} - b_{i} \rVert_{\infty} < \frac{\epsilon}{4} $ for $i = 1, \ldots, I$.
    Then,
    \begin{equation}
        \lVert \doublehat{f} - f \rVert_{\infty}
        \leq \lVert \doublehat{f} - \hat{f} \rVert_{\infty}
        + \lVert \hat{f} - f \rVert_{\infty}
        < \frac{\epsilon}{2} + \frac{\epsilon}{2} = \epsilon
        ,
    \end{equation}
    \thirdrev{
        since $
        \lvert \doublehat{f}(x, u) - \hat{f}(x, u) \rvert
        = \lvert
        \max_{1 \leq i \leq I} \{
            \langle \hat{a}_{i}(x), u \rangle + \hat{b}_{i}(x)
        \}
        - \max_{1 \leq i \leq I} \{
            \langle a_{i}(x), u \rangle + b_{i}(x)
        \}
        \rvert
        \leq \max_{1 \leq i \leq I} 
        \lvert
        \langle \hat{a}_{i}(x) - a_{i}(x), u \rangle
        + (\hat{b}_{i}(x) - b_{i}(x))
        \rvert
        \leq \frac{\epsilon}{4 \text{diam}(U)} \text{diam}(U) + \epsilon / 4
        = \epsilon / 2
        $,
        which concludes the proof.
    }
\end{proof}

\begin{corollary}[PLSE with ordinary universal approximators is a parameterized convex universal approximator]
    \label{cor:plse_with_ordinary_universal_approximator}
    Let $\mathcal{F}_{1}$ and $\mathcal{F}_{2}$ be ordinary universal approximators on $X$ to $\mathbb{R}^{m}$ and
    $\mathbb{R}$, respectively.
    Given parameterized convex continuous function $f: X \times U \to \mathbb{R}$,
    for any $\epsilon > 0$, there exist $I \in \mathbb{N}$ and $\overline{T} > 0$
    such that for all $T \in (0, \overline{T})$, there exist $\hat{a}_{i} \in \mathcal{F}_{1}$ and $\hat{b}_{i} \in \mathcal{F}_{2}$
    for $i = 1, \ldots, I$ such that
    $\lVert \doublehat{f} - f  \rVert_{\infty} < \epsilon$ where
    $\doublehat{f}(x, u) := T \log \left( \sum_{i=1}^{I} \exp \left( \frac{ \langle \hat{a}_{i}(x), u \rangle + \hat{b}_{i}}{T} \right) \right) \in \mathcal{F}_{T}^{PLSE}$.
\end{corollary}
\begin{proof}
    The proof can be shown from
    \autoref{thm:pma_with_ordinary_universal_approximator}
    and \autoref{thm:diff_PMA_and_PLSE}
    in a similar manner as the proof of
    \autoref{cor:plse_is_a_parameterized_convex_universal_approximator}
    and is omitted here.
\end{proof}

\thirdrev{
    In \autoref{thm:pma_with_ordinary_universal_approximator} and \autoref{cor:plse_with_ordinary_universal_approximator},
}
it is proven that PMA and PLSE networks are shape-preserving
universal approximators for parameterized convex continuous functions on the product of condition and decision spaces.
In practice, the decision space may depend on a given condition.
The following corollary shows that a simple modification can make
the above results applicable to conditional decision space settings.
\begin{corollary}[Extension to conditional decision space]
    Let $\secondrev{U_{cond}}: X \to \mathbb{R}^{m}$ be a mapping of conditional decision space such that
    $\secondrev{U_{cond}}(x)$ is convex compact for all $x \in X$.
    Suppose that there exists a convex compact subspace $U$ of $\mathbb{R}^{m}$ such that
    $\cup_{x \in X} \secondrev{U_{cond}}(x) \subset U$.
    Then,
    \autoref{thm:pma_is_a_parameterized_convex_universal_approximator}
    can be replaced by the conditional decision space setting.
\end{corollary}
\begin{proof}
    By \autoref{thm:pma_is_a_parameterized_convex_universal_approximator}, PMA is a parameterized convex universal approximator on $U$,
    A fortiori, PMA is a parameterized convex universal approximator on $\cup_{x \in X} \secondrev{U_{cond}}(x) \subset U$ from the assumption.
\end{proof}
Indeed,
it is straightforward to show that this extension
can easily be applied to other results, e.g.,
\thirdrev{
    \autoref{cor:plse_is_a_parameterized_convex_universal_approximator},
    \autoref{thm:pma_with_ordinary_universal_approximator},
    and \autoref{cor:plse_with_ordinary_universal_approximator}.
}

\thirdrev{
    The proofs of \autoref{thm:eps_selection_of_subdiff_mapping}
    and \autoref{thm:pma_is_a_parameterized_convex_universal_approximator}
    are the main contributions of this study although
    they are described in \hyperref[sec:proof_of_eps_selection_of_subdiff_mapping]{Appendix A}
    and \hyperref[sec:proof_of_pma_is_a_parameterized_convex_universal_approximator]{Appendix B},
    respectively, for the readability.
    The proofs can be summarized as follows.
    \hyperref[sec:proof_of_pma_is_a_parameterized_convex_universal_approximator]{Appendix B}
    proves \autoref{thm:pma_is_a_parameterized_convex_universal_approximator}
    first under the assumptions of Lipschitzness and nonempty interior of $U$
    and completes the proof by relaxing the assumptions.
    \autoref{thm:eps_selection_of_subdiff_mapping} provides
    tools for the relaxation of the Lipschitzness assumption;
    \hyperref[sec:proof_of_eps_selection_of_subdiff_mapping]{Appendix A}
    builds continuous selections based on multivalued analysis.
}

\begin{remark}[Comparison to the existing convex universal approximators]
    Since the projection of convex functions is also convex,
    the existing convex universal approximators, e.g., MA and LSE networks \cite{calafioreLogSumExpNeuralNetworks2020},
    are also applicable to decision-making problems.
    That is, with appropriate dimension modification,
    Eq. \eqref{eq:ma_and_lse} changes to
    \begin{equation*}
        \begin{split}
            f^{\text{MA}}(x, u)
        &= \max_{1 \leq i \leq I} \left( \langle a_{i}, z \rangle + b_{i} \right),
        \\
        f^{\text{LSE}}(x, u)
        &= T \log \left( \sum_{i=1}^{I} \exp \left( \frac{\langle a_{i}, z \rangle + b_{i}}{T} \right) \right),
        \end{split}
    \end{equation*}
    where $z = [x^{\intercal}, u^{\intercal}]^{\intercal}$.
    Examples of this approach include
    finite-horizon Q-learning using LSE
    \cite{calafioreEfficientModelFreeQFactor2020}.
    Compared to this projection approach,
    PMA and PLSE have several advantages:
    i) PMA and PLSE are not required to restrict the condition space $X$
    to be a convex compact space, while MA and LSE are,
    and ii) PMA and PLSE can be constructed by utilizing deep networks,
    e.g., multilayer FNN,
    for $\hat{a}_{i}$'s and $\hat{b}_{i}$'s in
    \autoref{thm:pma_with_ordinary_universal_approximator}
    and \autoref{cor:plse_with_ordinary_universal_approximator}.
Building a network with deep networks would be
    practically attractive \firstrev{considering the successful application}
    of deep learning in numerous fields
    \cite{chenNeuralOrdinaryDifferential2018,goodfellowGenerativeAdversarialNetworks2014,vaswaniAttentionAllYou2017}.
\end{remark}

\begin{remark}[Comparison between PMA and PLSE]
    As LSE networks can be viewed as smoothed MA networks
    by replacing the pointwise supremum with the log-sum-exp operator,
    PLSE networks can likewise be viewed as smoothed PMA networks with respect to
    the decision variable $u$.
    The choice of networks may depend on
    tasks, domain knowledge, convex optimization solvers, etc.
\end{remark}

\begin{remark}[Data normalization]
    Normalization of data may be critical for
    the training and inference performance of PLSE.
    Note that LSE can also be normalized by its temperature parameter
    as if all LSE networks have the same temperature,
    usually set to be one, i.e., $T = 1$
    \cite{calafioreLogSumExpNeuralNetworks2020}.
    It should be pointed out that the temperature normalization
    is also applicable for PLSE in the same manner as in LSE.
\end{remark}

\section{Numerical Simulation}
\label{sec:numerical_simulation}

\begin{table*}[b!]
    \caption{Simulation settings}
    \label{table:simulation_settings}
    \centering
    \begin{tabular}{|| c | c | c ||}
        \hline
        Parameters & Values & Remarks \\
        \hline \hline
        \multirow{2}{*}{Hidden layer width} & $(64, 64)$ & FNN, PMA, PLSE, $u$-path of PICNN\\
                                                  & $(64, 64, 64)$ & $x$-path of PICNN \\
        \hline
        No. of parameter pairs, $I$ & $30$ & MA, LSE, PMA, and PLSE (Eqs. \eqref{eq:ma_and_lse}, \eqref{eq:pma_and_plse}) \\
        \hline
        Activation function & $\max(0.01x, x)$ & LeakyReLU (elementwise)\\
        \hline
        Temperature, $T$ & 0.1 & PLSE (Eq. \eqref{eq:pma_and_plse}) \\
        \hline
        Training epochs & $100$ & \\
        \hline
        Optimizer & ADAM \cite{kingmaAdamMethodStochastic2017} with learning rate of $10^{-3}$ & with projection for PICNN \cite{amosInputConvexNeural2017a} \\
        \hline
        Number of data points, $d$ & $5\text{,}000$ & train = $90$:$10$ \\
        \hline
        Parameter initialization & $w \sim U \left[-\frac{\sqrt{6}}{\sqrt{n_{\text{in}}+n_{\text{out}}}}, \frac{\sqrt{6}}{\sqrt{n_{\text{in}} + n_{\text{out}}}}\right]$ & Xavier initialization$^{1}$ \cite{glorotUnderstandingDifficultyTraining2010}  \\
        \hline
        \multicolumn{3}{l}{%
            \begin{minipage}{10cm}%
                $^{1}$: $n_{\text{in}} = 1$ for $a_{i}$'s and $b_{i}$'s in MA and LSE (Eq. \eqref{eq:ma_and_lse}).
            \end{minipage}%
        }\\
    \end{tabular}
\end{table*}
In this section,
a numerical simulation of the function approximation is performed to demonstrate
the proposed approximators' approximation capability, optimization accuracy,
and \firstrev{solving time of optimization for a given condition to perform decision-making}.
For the simulation of the proposed parameterized convex approximators,
a Julia \cite{bezansonJuliaFreshApproach2017} package,
ParametrisedConvexApproximators.jl\footnote{\href{https://github.com/JinraeKim/ParametrisedConvexApproximators.jl}{https://github.com/JinraeKim/ParametrisedConvexApproximators.jl}},
is \secondrev{developed in this study}.
\firstrev{All simulations were performed on a desktop with an AMD Ryzen\texttrademark{} 9 5900X and Julia v1.7.1.}

Several approximators will be compared:
FNN, MA, LSE, PICNN, PMA, and PLSE.
FNN is the most widely used class of neural networks and was proven to be an ordinary universal approximator for certain architectures.
MA and LSE are convex universal approximators \cite{calafioreLogSumExpNeuralNetworks2020}.
PICNN is a parameterized convex approximator proposed for decision-making problems,
mainly motivated by energy-based learning to perform inference via optimization \cite{amosInputConvexNeural2017a}.
The main characteristics of PICNN include its deep architecture recursively constructed
with \firstrev{two paths, i.e., $x$-path and $u$-path}.
Note that it has not been shown that PICNN is a parameterized convex universal approximator.

\begin{figure}[t!]
    \centering
    \includegraphics[width=.99\linewidth]{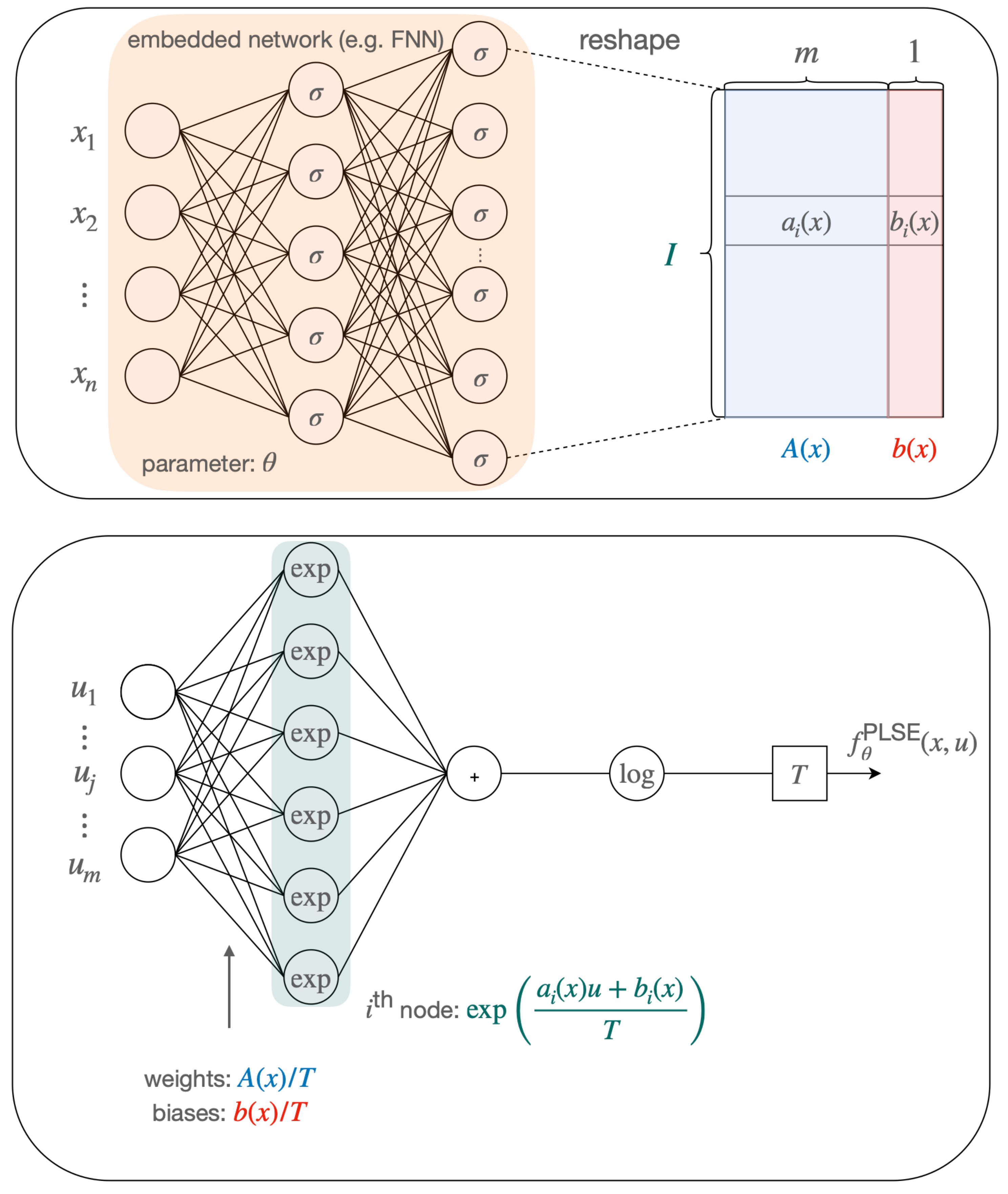}
    \caption{Illustration of the implemented PLSE network}
    \label{fig:plse_structure}
\end{figure}

The \firstrev{following} target function $f$ is chosen as a parameterized convex function,
\begin{equation}
    f(x, u) = -\frac{1}{2 n} x^{\intercal}x + \frac{1}{2 m} u^{\intercal}u
    .
\end{equation}
\autoref{fig:target_function} shows the target function for $(n, m) = (1, 1)$.

\begin{figure}[t!]
    \centering
    \includegraphics[width=.70\linewidth]{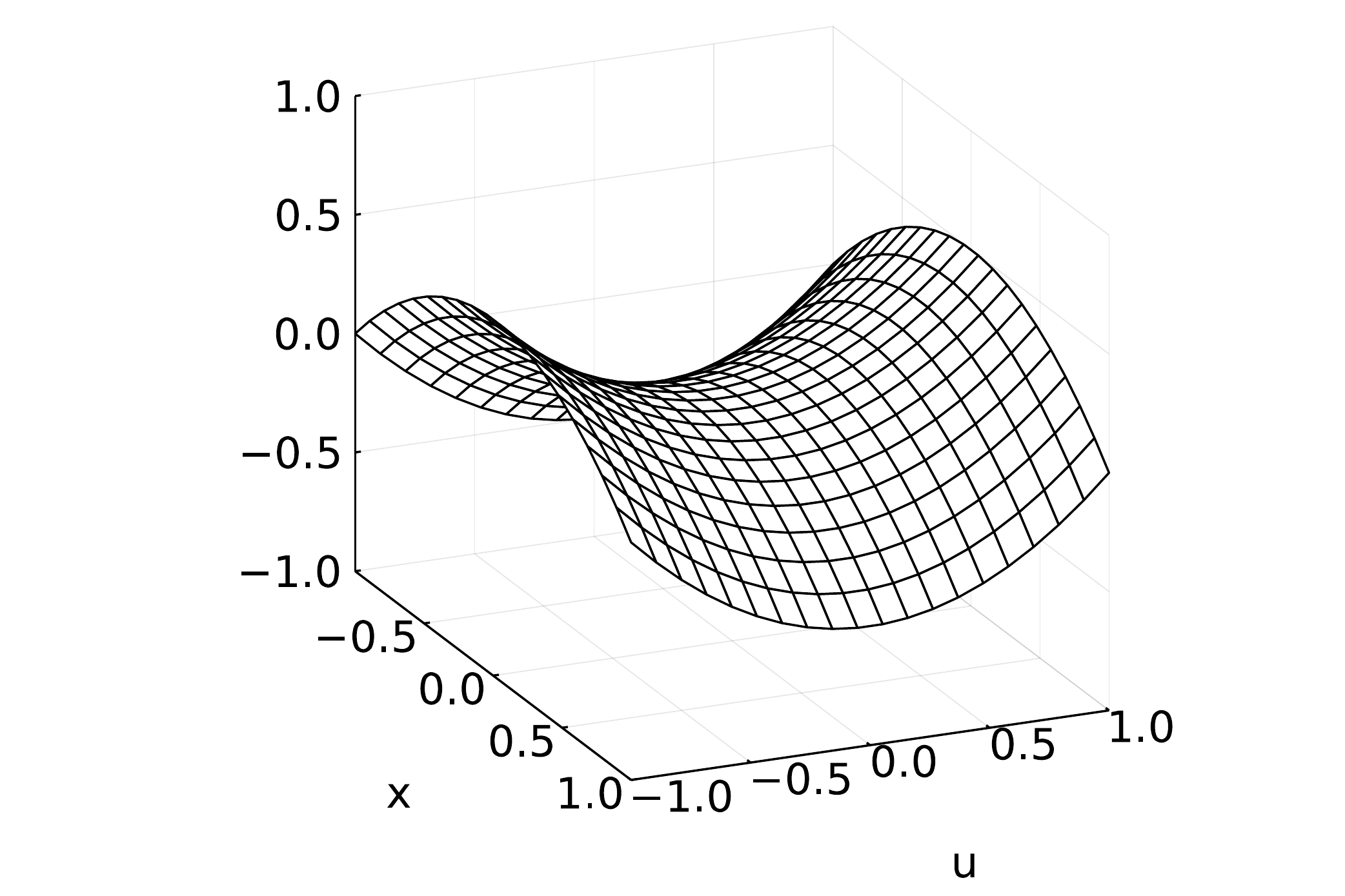}
    \caption{Target function}
    \label{fig:target_function}
\end{figure}

\begin{figure*}[t!]
    \centering
    \subfloat[FNN]{% 
        \includegraphics[width=.33\linewidth]{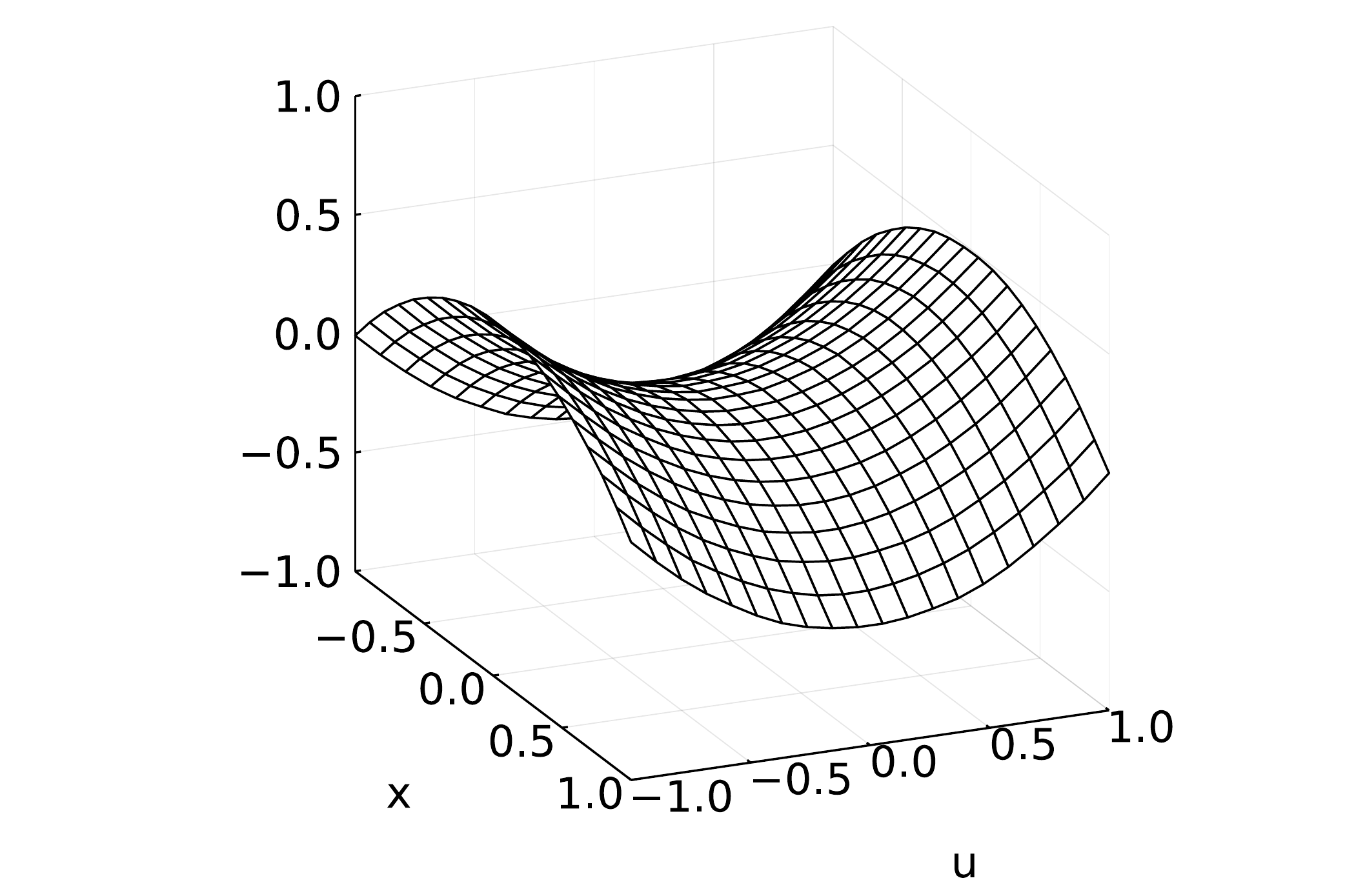}
    }
    \subfloat[MA]{% 
        \includegraphics[width=.33\linewidth]{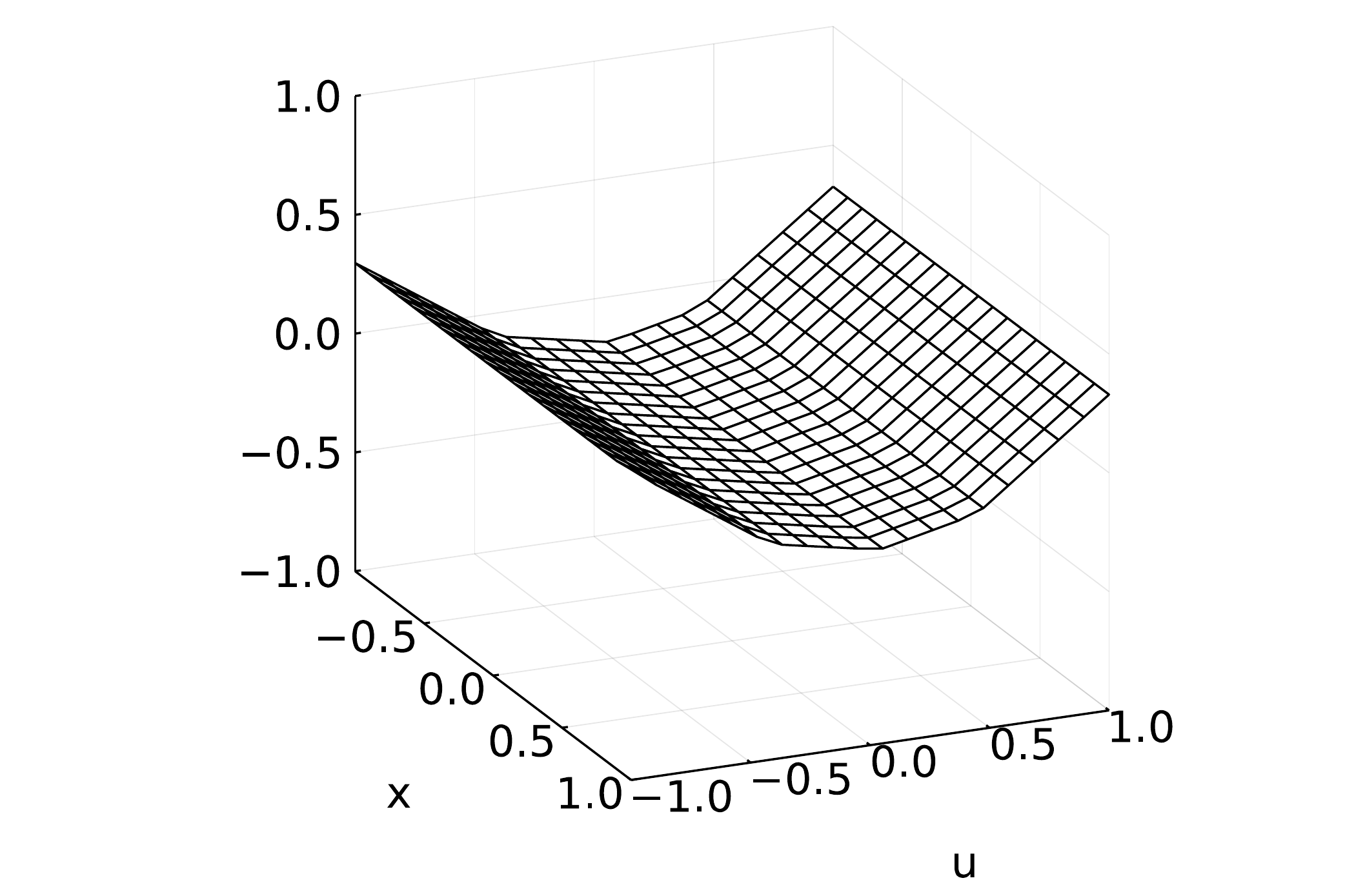}
    }
    \subfloat[LSE]{% 
        \includegraphics[width=.33\linewidth]{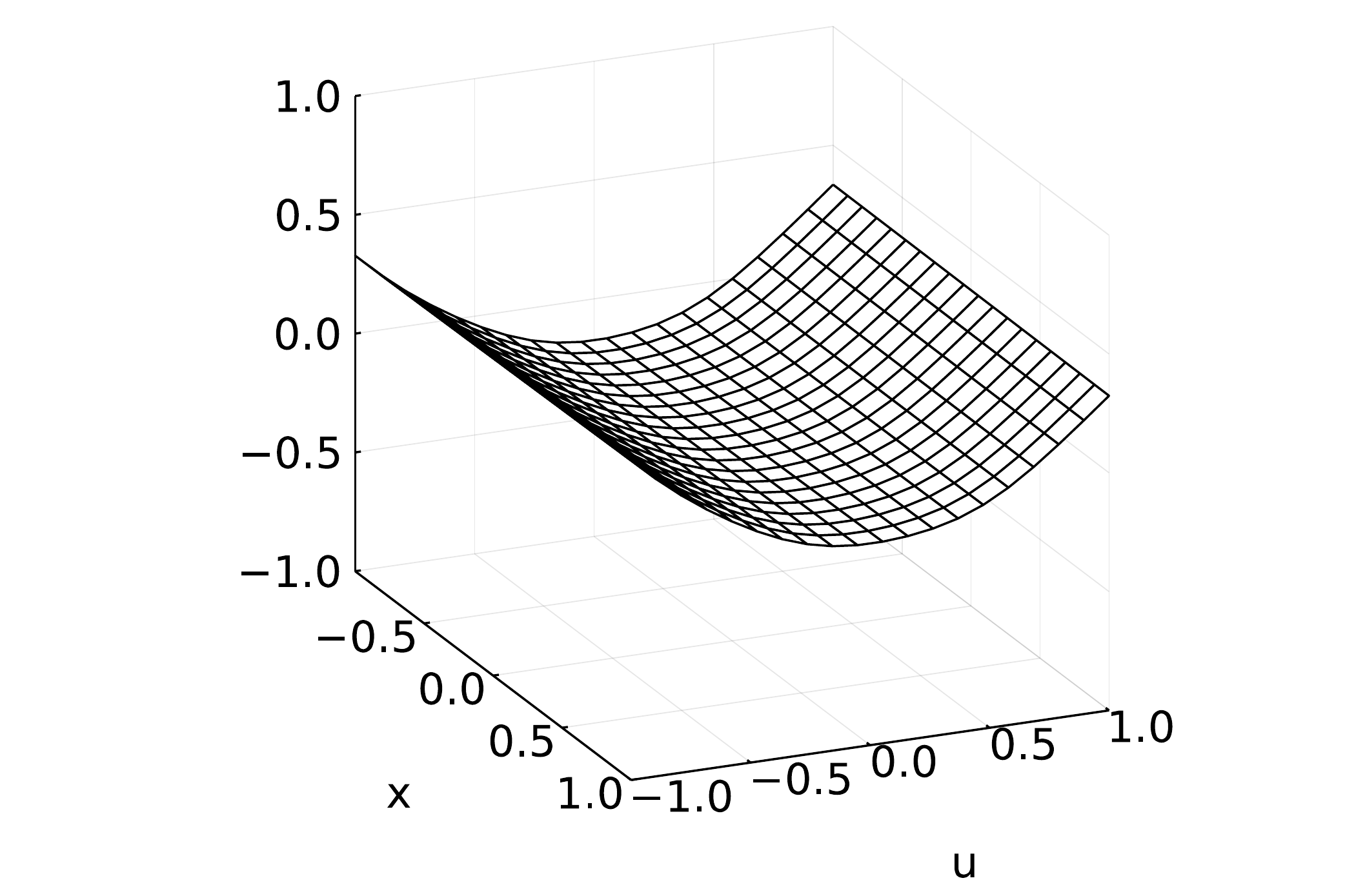}
    }

    \subfloat[PICNN]{% 
        \includegraphics[width=.33\linewidth]{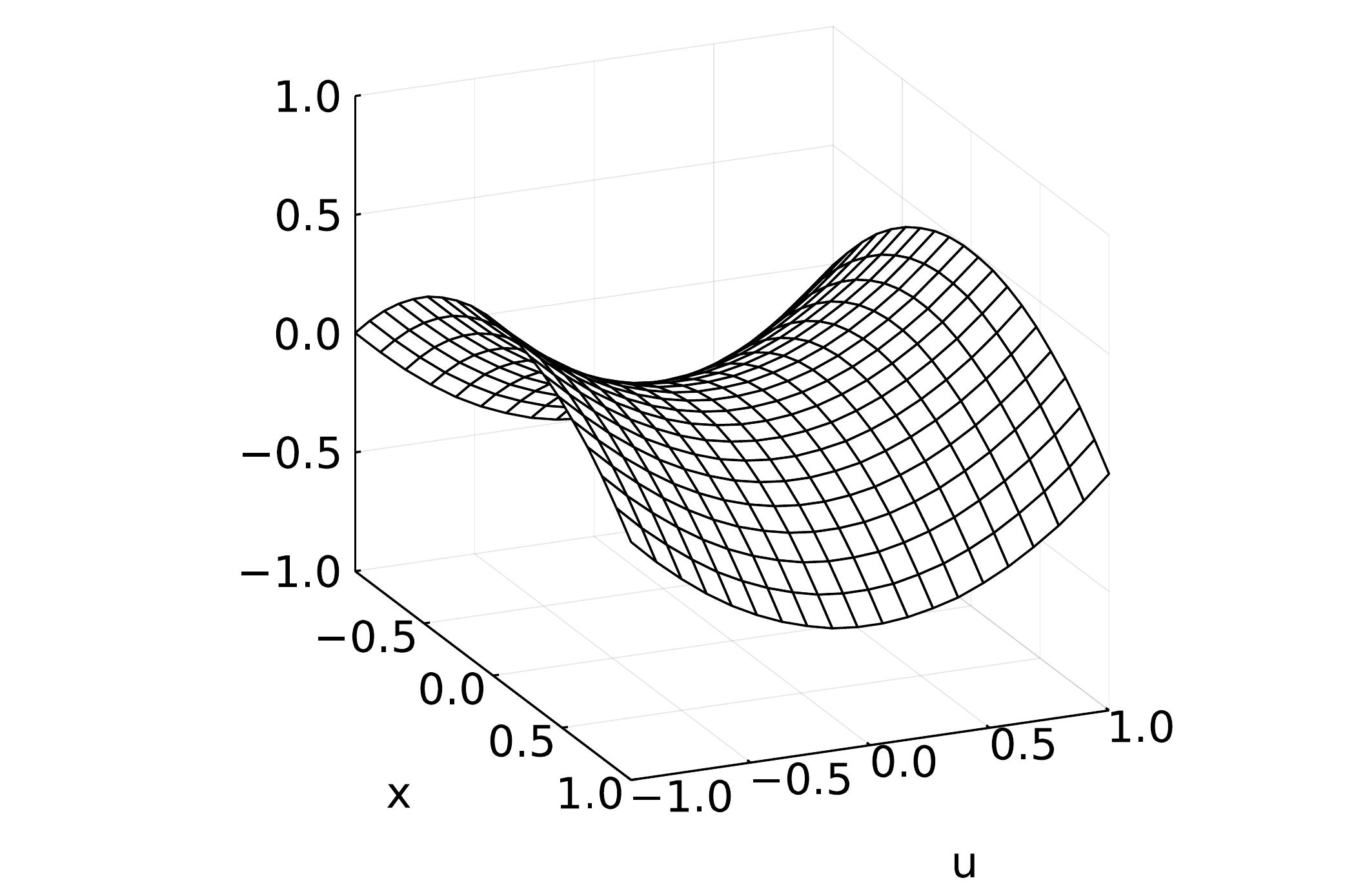}
    }
    \subfloat[PMA]{% 
        \includegraphics[width=.33\linewidth]{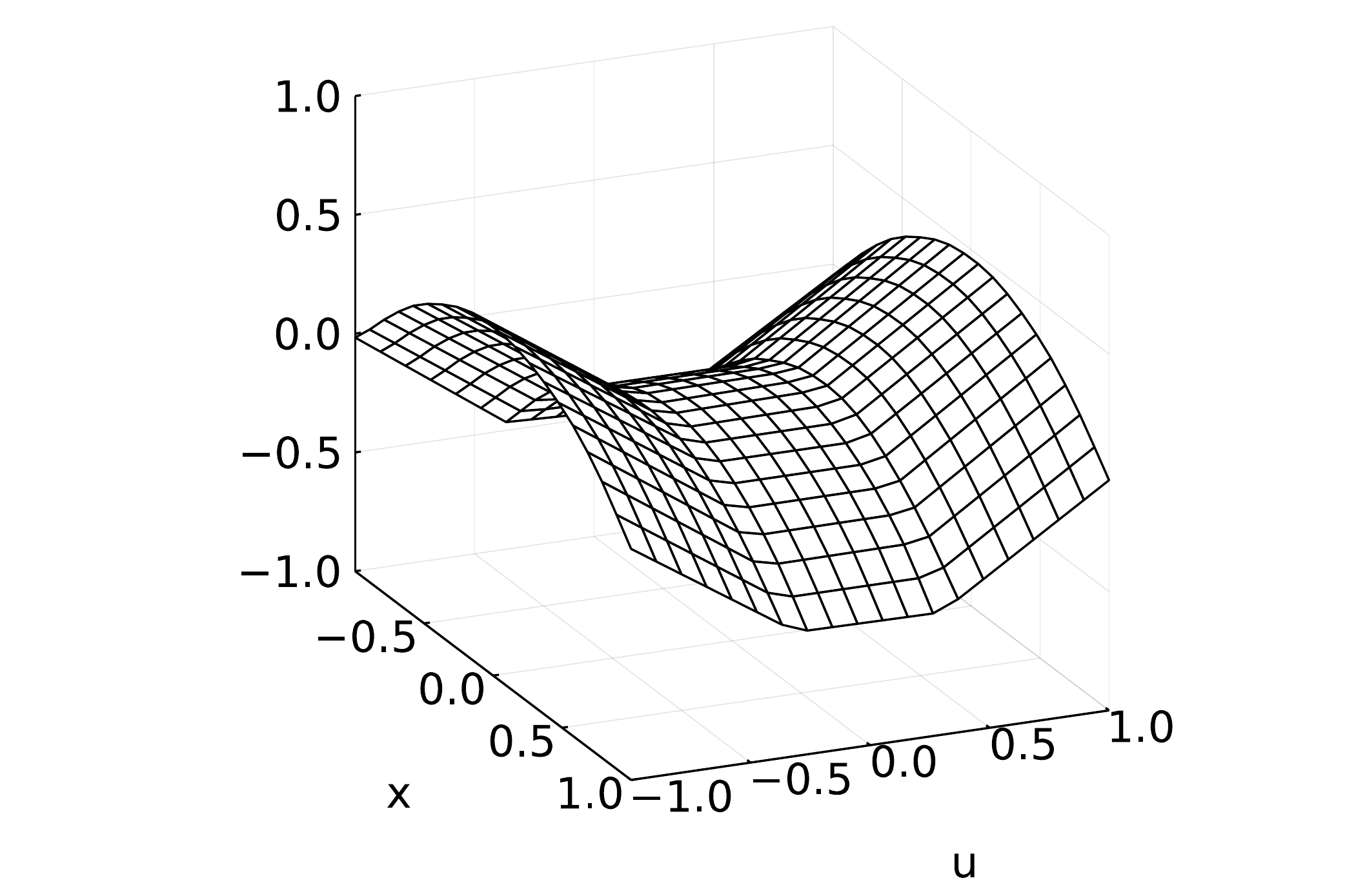}
    }
    \subfloat[PLSE]{% 
        \includegraphics[width=.33\linewidth]{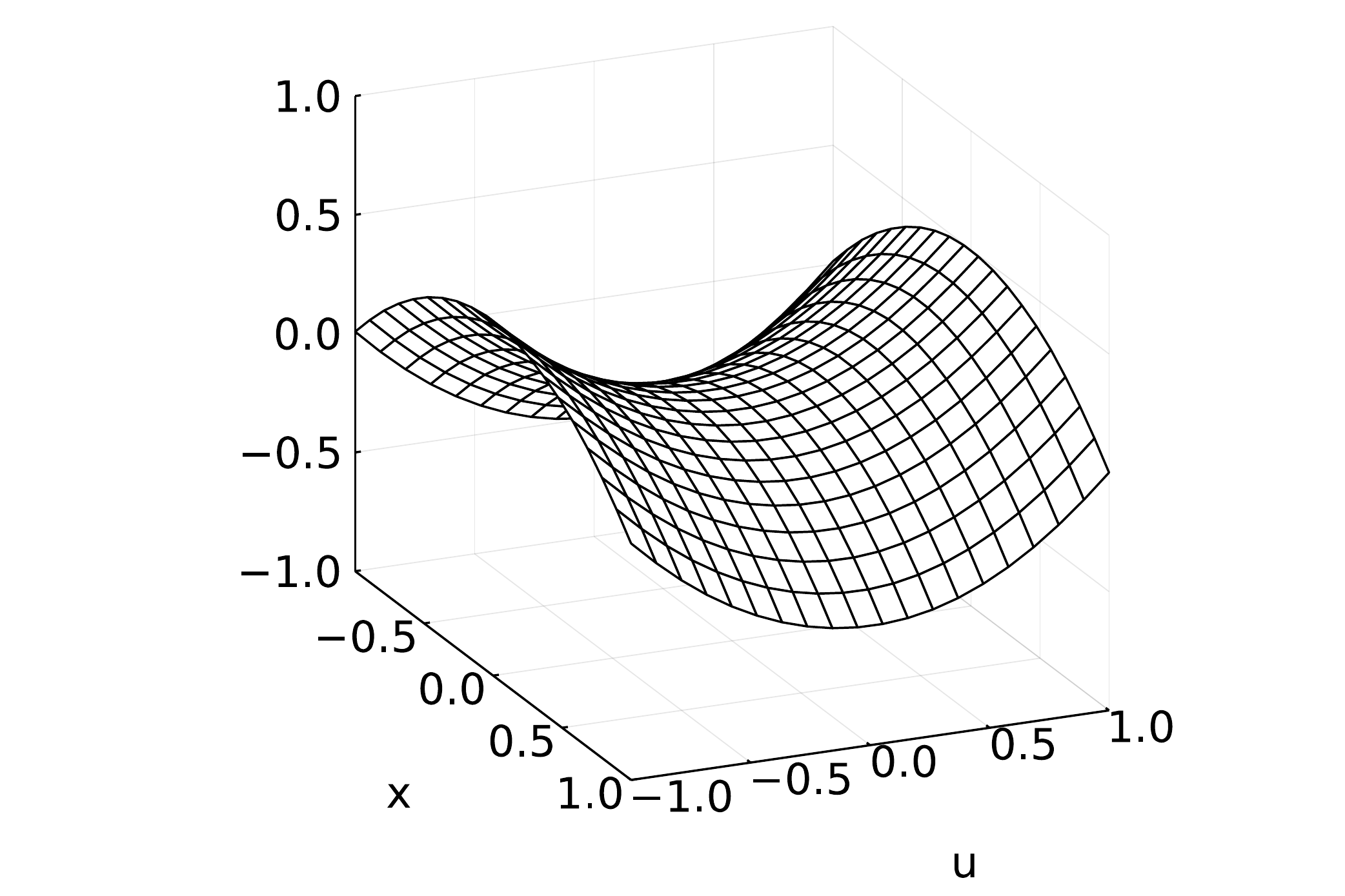}
    }
    \caption{Trained approximators}
    \label{fig:trained_approximators}
\end{figure*}

\begin{figure*}[t!]
    \centering
    \subfloat[Minimizer error $(n, m = 1, 1)$]{% 
        \includegraphics[width=.33\linewidth]{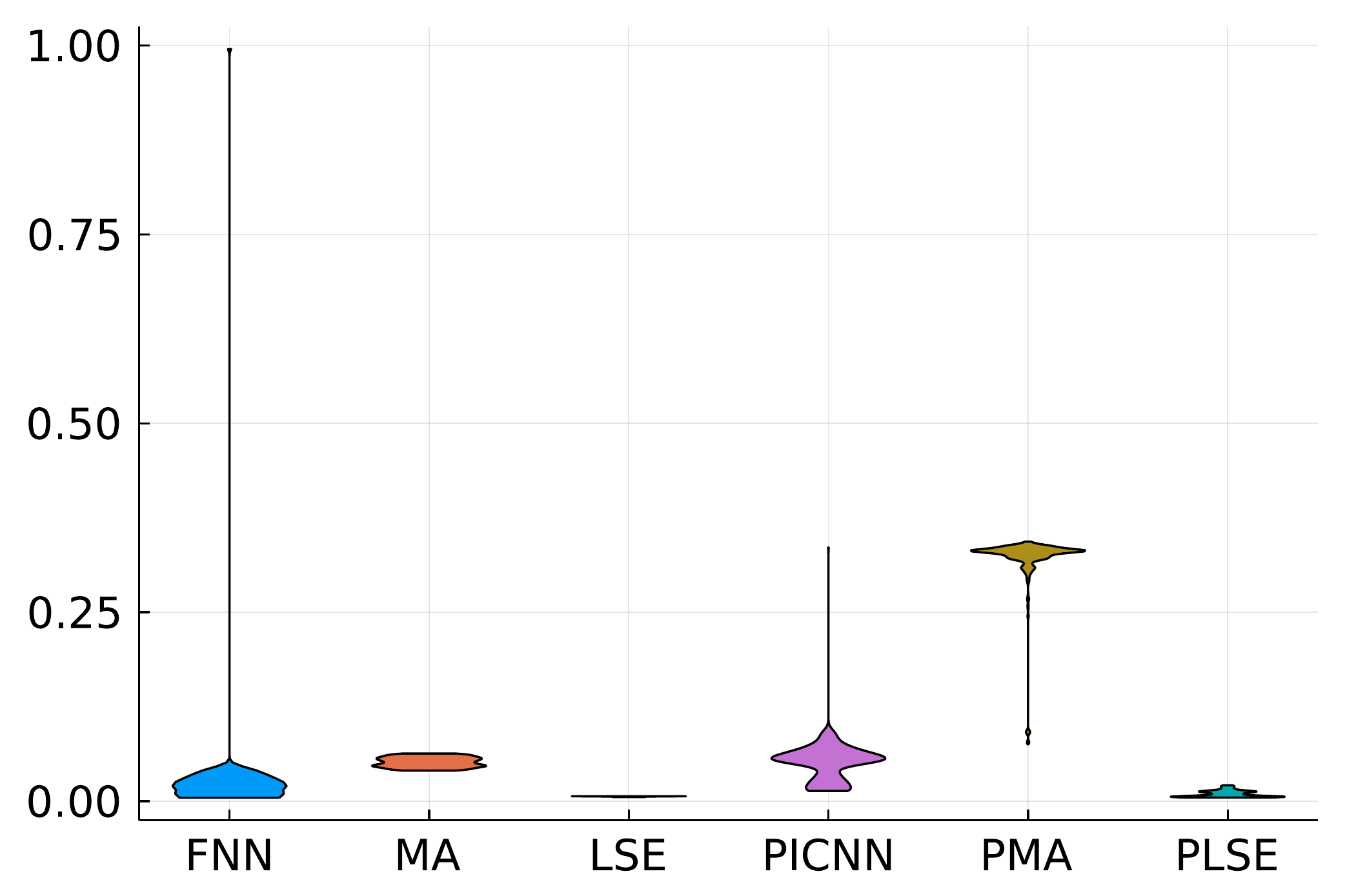}
    }
    \subfloat[Minimizer error $(n, m = 61, 20)$]{% 
        \includegraphics[width=.33\linewidth]{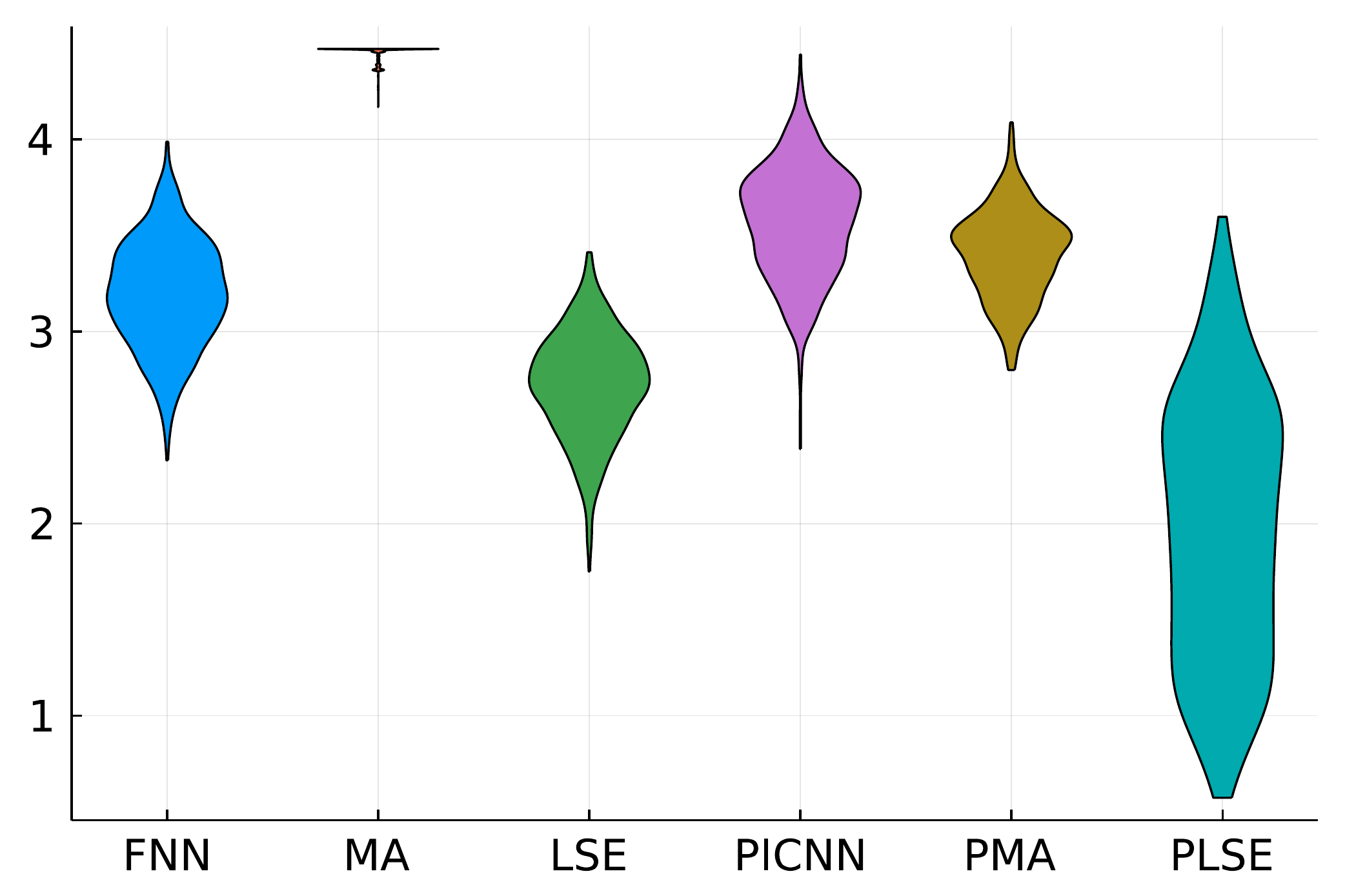}
    }
    \subfloat[Minimizer error $(n, m = 376, 17)$]{% 
        \includegraphics[width=.33\linewidth]{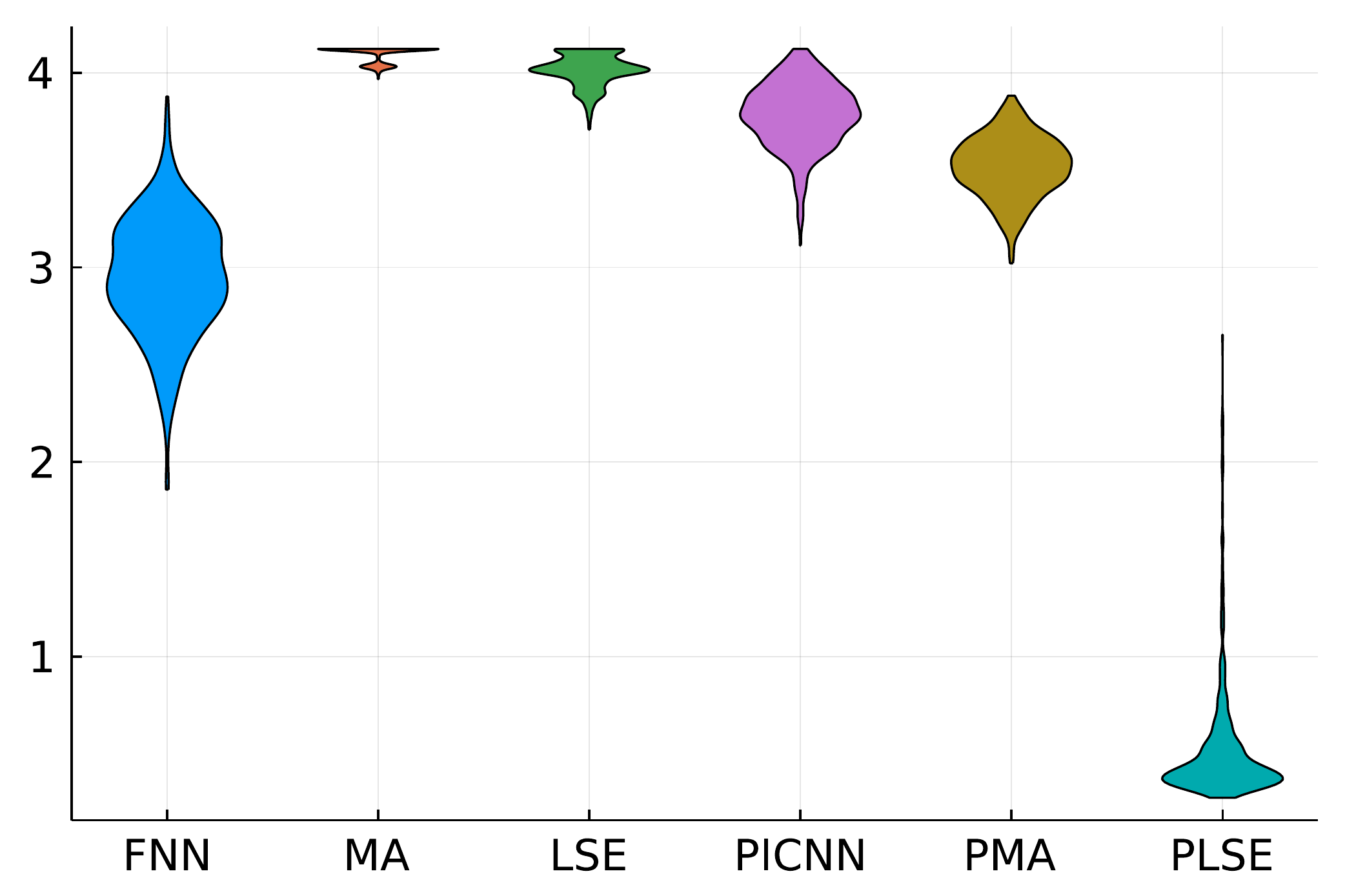}
    }

    \subfloat[Optimal value error $(n, m = 1, 1)$]{% 
        \includegraphics[width=.33\linewidth]{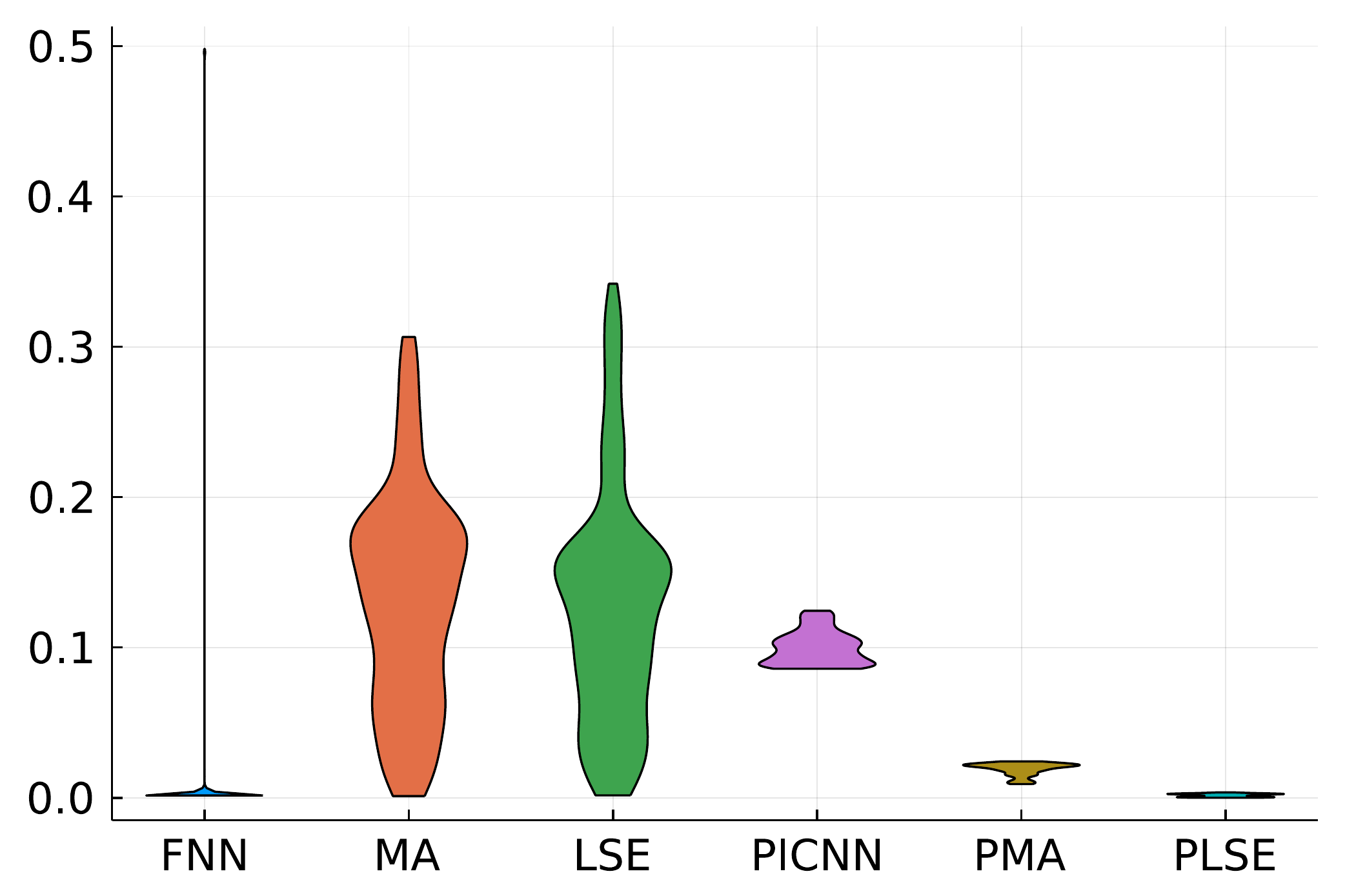}
    }
    \subfloat[Optimal value error $(n, m = 61, 20)$]{% 
        \includegraphics[width=.33\linewidth]{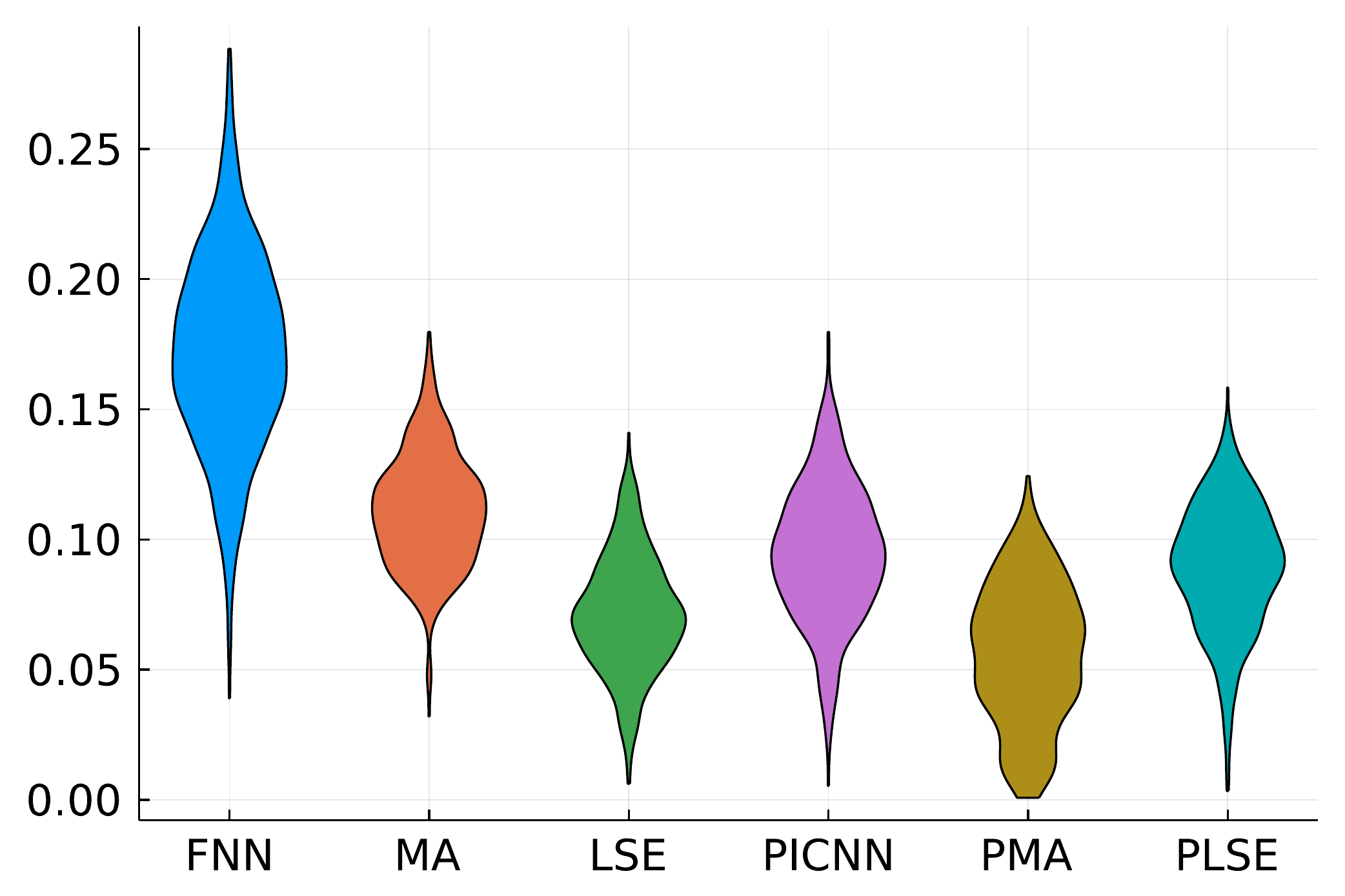}
    }
    \subfloat[Optimal value error $(n, m = 376, 17)$]{% 
        \includegraphics[width=.33\linewidth]{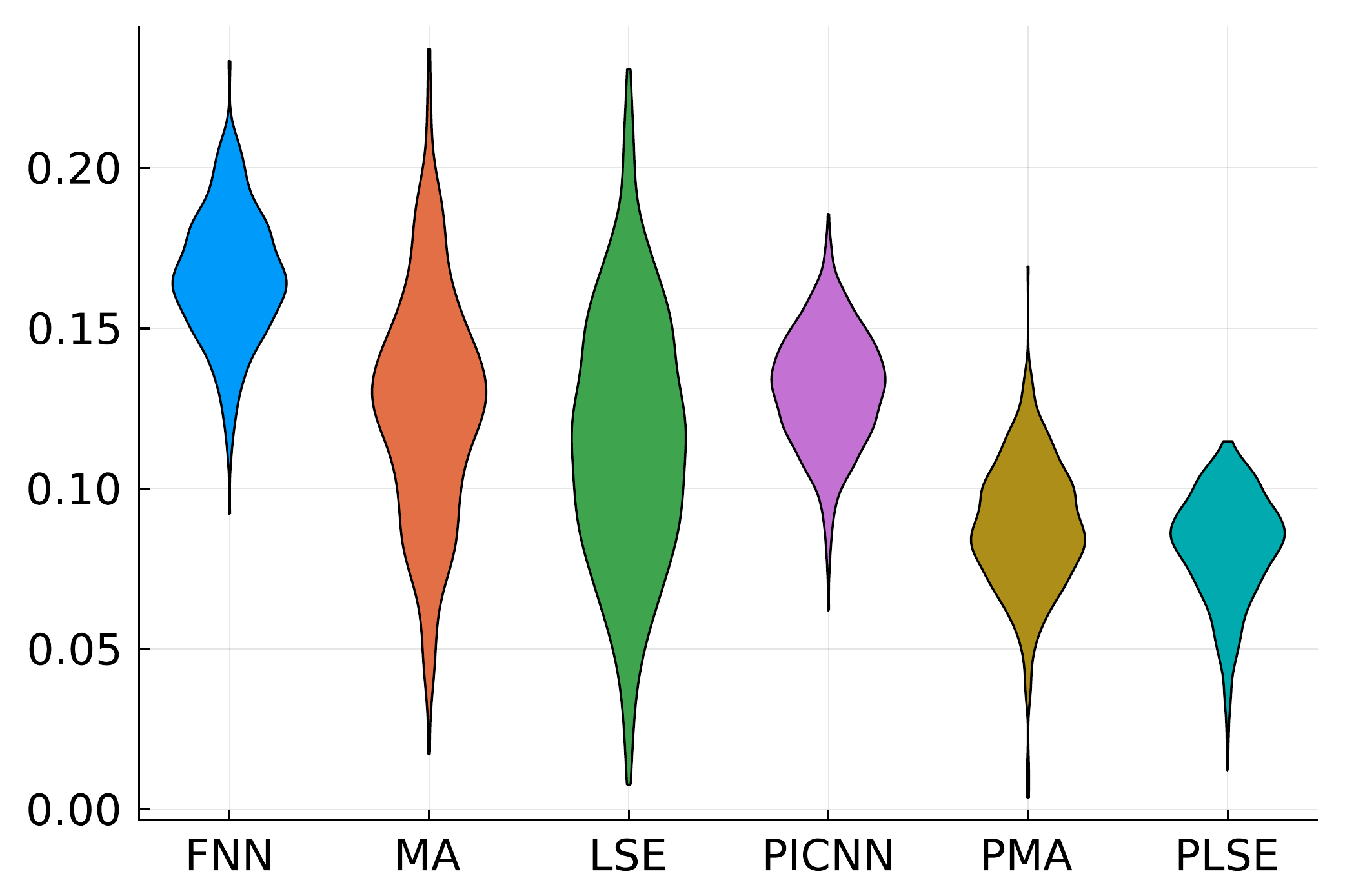}
    }
    \caption{Violin plot of minimizer and optimal value errors}
    \label{fig:error}
\end{figure*}

Data points $\{ (x, u)_{i} \}_{i = 1}^{d}$ are uniformly randomly sampled within
\firstrev{the box or hypercube for high-dimensional cases}
of $[-1, 1]^{n} \times [-1, 1]^{m}$,
where $d \in \mathbb{N}$ denotes the number of data points.
The data points are split \firstrev{into} two groups, train and test data, in a ratio of $90$:$10$.
Each approximator $\hat{f}$ is trained via supervised learning with the 
ADAM optimizer \cite{kingmaAdamMethodStochastic2017}.
Simulation settings are summarized in \autoref{table:simulation_settings}.
\firstrev{
    FNN is constructed as a fully connected feedforward neural network
    with input $n+m$ nodes, hidden layer nodes (denoted as hidden layer width in \autoref{table:simulation_settings}),
    and output $1$ node.
    \thirdrev{
        For PMA and PLSE,
        a fully connected FNN \secondrev{is constructed}
        with input $n$ nodes,
        hidden layer nodes, and output $(m+1)I$ nodes as an embedded network.
        The output of the embedded network is split into matrices whose sizes are $I \times m$ and $I \times 1$, respectively.
        \secondrev{The} $i$-th row of each matrix corresponds to $a_{i}(x)$ and $b_{i}(x)$ in Eq. \eqref{eq:pma_and_plse}, for $i = 1, \ldots, I$.
        The illustration of the PLSE network with the embedded network is illustrated in \autoref{fig:plse_structure}.
    }
    For the construction of PICNN, $x$- and $u$-paths' hidden layer nodes are used \cite{amosInputConvexNeural2017a}.
}
\firstrev{
    Each simulation is performed with various input-output dimensions of $(n, m) \in \{(1, 1), (61, 20), (376, 17)\}$.
    Note that $(n, m) = (1, 1)$ is for 3D visualization,
    and others are borrowed from OpenAI gym environments \cite{brockmanOpenAIGym2016}
    considering future applications to high-dimensional dynamical systems:
    $(n, m) = (61, 20)$ and $(n, m) = (376, 17)$ from the dimensions of observation and action spaces of
    \texttt{HandManipulateBlock-v0} and \texttt{Humanoid-v2} from OpenAI gym, respectively.
    Note that
    \texttt{HandManipulateBlock-v0} is an environment for guiding a block on a hand to a randomly chosen goal orientation in three-dimensional space,
    and \texttt{Humanoid-v2} is an environment for making a three-dimensional bipedal robot walk forward as fast as possible without falling over.
    From the fact that the observation and action spaces of most dynamical systems
    range in dimension from $1$ and $1$ to \secondrev{$376$} and $20$, respectively,
    each dynamical system represents the dynamical system with
    i) medium- $(61)$ and high-dimensional $(20)$ state and action variables
    and ii) high- ($376$) and slightly high-dimensional $(17)$ state and action variables.
}

\autoref{fig:trained_approximators} shows the illustration of trained approximators,
only for $(n, m) = (1, 1)$ for 3D visualization.
As shown in \autoref{fig:trained_approximators},
most approximators can approximate a given target \firstrev{function, which is not convex but parameterized convex,}
for $(n, m) = (1, 1)$ as its dimension is very low,
while MA and LSE cannot because they are convex universal approximators.
Note that MA and PMA have piecewise-linear \firstrev{parts, especially along the $u$-axis,} due to
their maximum operator and linear combination.
Compared to others, FNN, PICNN, and PLSE approximate the target function smoothly.
This low-dimensional visualization indicates that MA and LSE may be restrictive
for function approximation of a larger class of functions, e.g., parameterized convex functions,
\thirdrev{
    and FNN, PICNN, and PLSE can approximate the smooth target function very well.
}

\begin{table*}[t!]
    \caption{\firstrev{Numerical optimization result}}
    \label{table:solve_time}
    \centering
    \begin{tabular}{|| c | c c c c c c ||}
        \hline
        Dimensions & FNN & MA & LSE & PICNN & PMA & PLSE \\
        \hline \hline
        \multicolumn{7}{||c||}{\firstrev{Mean solving time [s]$^{1}$}} \\
        \hline
        $(n, m) = (1, 1)$ & $0.0008$ & $\mathbf{0.0005}$ & $0.0389$ & $0.0108$ & $0.0011$ & $0.0065$ \\
        \firstrev{$(n, m) = (61, 20)$} & $0.0355$ & $\mathbf{0.0096}$ & $0.0880$ & $0.1236$ & $0.0186$ & $0.0111$ \\
        $(n, m) = (376, 17)$ & $0.0313$ & $0.0110$ & $0.2436$ & $0.0737$ & $0.0163$ & $\mathbf{0.0070}$ \\
        \hline
        \multicolumn{7}{||c||}{Mean of minimizer errors (2-norm)} \\
        \hline
        $(n, m) = (1, 1)$ & $0.0340$ & $0.0518$ & $\mathbf{0.0063}$ & $0.0512$ & $0.3206$ & $0.0097$ \\
        \firstrev{$(n, m) = (61, 20)$} & $3.1932$ & $4.4515$ & $2.7190$ & $3.5730$ & $3.4014$ & $\mathbf{1.9736}$ \\
        $(n, m) = (376, 17)$ & $2.9522$ & $4.0974$ & $4.0131$ & $3.7825$ & $3.5163$ & $\mathbf{0.5204}$ \\
        \hline
        \multicolumn{7}{||c||}{Mean of optimal value errors (absolute)} \\
        \hline
        $(n, m) = (1, 1)$ & $0.0084$ & $0.1315$ & $0.1268$ & $0.1001$$^{2}$ & $0.0190$ & $\mathbf{0.0017}$ \\
        \firstrev{$(n, m) = (61, 20)$} & $0.1720$ & $0.1111$ & $0.0714$ & $0.0946$ & $\mathbf{0.0564}$ & $0.0892$ \\
        $(n, m) = (376, 17)$ & $0.1660$ & $0.1256$ & $0.1154$ & $0.1312$ & $0.0879$ & $\mathbf{0.0815}$ \\
        \hline
        \multicolumn{7}{l}{%
            \begin{minipage}{10cm}%
                $^{1}$: total solving time divided by the number of test data $d$
                \\
                \firstrev{$^{2}$: invalid values ($-\infty$ or $+\infty$) obtained from the solver for $2$ cases}
            \end{minipage}%
        }\\
    \end{tabular}
\end{table*}

Numerical optimization is performed with the trained approximators to verify
that the proposed approximators are well suited for optimization-based inference considering various $(n, m)$ pairs.
The condition variables in the test data (unseen data) are used for the numerical optimization.
For the decision variable optimization, a box constraint is imposed \firstrev{on the data point sampling}, that is, $u \in [-1, 1]^{m}$.
\firstrev{
    For (parameterized) convex approximators, a splitting conic solver (SCS) convex optimization solver is used, SCS.jl v0.8.1 \cite{scs},
    with a disciplined convex programming (DCP) package, Convex.jl v0.14.18 \cite{convexjl}.
    For nonconvex approximators,
    FNN in this case,
    the interior-point Newton method is used for optimization using Optim.jl v1.5.0 \cite{kmogensenOptimMathematicalOptimization2018}.
}

Figure \ref{fig:error} \firstrev{shows} the violin plot of the 2-norm of minimizer error and the absolute value of the optimal value error.
FNN shows a precise minimizer and optimal value estimation performance for low dimensions of $(n, m) = (1, 1)$ except for a few cases.
However, FNN's estimation performance is significantly \firstrev{degenerated} for high dimensions of $(n, m) = (61, 20)$ and $(376, 17)$.
Due to its nonconvexity, the optimization solver finds poor minimizers and optimal values for high-dimensional cases.
MA and LSE \firstrev{show} slightly poor estimation performance.
Given that they \firstrev{show} poor approximation capability for parameterized convex functions, as shown in \autoref{fig:trained_approximators},
if the target function were not symmetric, the minimizer and optimal value estimation performance would be worse.
PICNN's performance seems better as the dimension increases, but for all cases including the high-dimensional cases of $(n, m) = (376, 17)$,
the performance of PLSE is superior to that of PICNN.
Note that MA and PMA \firstrev{show} worse minimizer and optimal value estimation performance compared to LSE and PLSE in most cases,
respectively,
due to their inherited non-smoothness from piecewise-linear construction.
It should be pointed out that PLSE shows the near-minimum error in terms of both minimizer error and optimal value error.

\autoref{table:solve_time} shows the mean values of \firstrev{solving time for optimization}, minimizer error, and optimal value, averaged with test data (unseen data).
For the low-dimensional case of $(n , m) = (1, 1)$,
\firstrev{the mean solving times of FNN and MA are the smallest.}
As the dimension increases,
most approximators' \firstrev{solving times} are not scalable;
however, PLSE shows the scalable \firstrev{solving time} for $(n, m) = (61, 20)$ and $(376, 17)$.
This tendency is also similar for minimizer and optimal value errors.
For most cases, PLSE shows the near-minimum mean values of minimizer and optimal value errors.
It should be clarified in this simulation study that
all approximators have similar hyperparameters of network architecture
although similar hyperparameter settings \firstrev{do} not imply a similar number of network parameters for each approximator.
For example, the same $I$ for MA, LSE, PMA, and PLSE gives very different numbers of network parameters.
The reason is related to how convex optimization solvers and DCP packages work:
Most DCP packages generate slack variables and related constraints to transform the original problem
into a corresponding conic problem.
For instance, consider the following optimization problem:
\begin{equation}
    \begin{split}
        \text{Minimize} & \quad \max_{1 \leq i \leq I} u_{i}
        \\
        \text{subject to} & \quad u_{\min} \leq u \leq u_{\max},
    \end{split}
\end{equation}
where $u = [u_{1}, \ldots, u_{I}]^{\intercal}$ is the optimization variable,
and the inequality constraint is elementwise.
Then, the above problem can be transformed into an equivalent problem as
\begin{equation}
    \begin{split}
        \text{Minimize} & \quad t
        \\
        \text{subject to} & \quad u_{\min} \leq u \leq u_{\max},
        \\
                          & \quad u \leq t\mathds{1},
    \end{split}
\end{equation}
where $\mathds{1} \in \mathbb{R}^{I}$ denotes a vector whose elements are one.
Similarly, in the case of MA,
\firstrev{
    it generates one slack variable with $I$ additional constraints.
    For example, to make MA have a similar number of network parameters in the above simulation,
    $I$ should be up to approximately $10\text{,}000 \sim 100\text{,}000$,
    which significantly slows the solver.
}
In this regard, MA and LSE cannot easily increase the number of network parameters
because the number of their network parameters depends only on $I$.
In contrast,
PMA and PLSE can have many network parameters while fixing $I$,
which maintains the sufficiently fast \firstrev{solving time}.

\firstrev{
    In summary,
    the results from the numerical simulation
    support
} that
PLSE can approximate a given target function from low- to high-dimensional cases
with considerably small minimizer and optimal value errors and fast \firstrev{solving time}
compared to other existing approximators.
The existing approximators include
an ordinary universal approximator, FNN, convex universal approximators,
MA and LSE, and a parameterized convex approximator, PICNN.
\thirdrev{
    It should be pointed out that PICNN, a parameterized convex approximator
    without guarantees of universal approximation theorem,
    may not have satisfactory approximation performance even with a simple target function.
}

\section{Conclusion}
\label{sec:conclusion}
In this paper, parameterized convex approximators were proposed, namely,
the parameterized max-affine (PMA) and parameterized log-sum-exp (PLSE) networks.
PMA and PLSE generalize the existing convex universal approximators,
max-affine (MA) and log-sum-exp (LSE) networks, respectively,
for approximation of parameterized convex functions
by replacing network parameters with continuous functions.
It was proven that PMA and PLSE are shape-preserving universal approximators,
i.e., they are parameterized convex and can approximate any parameterized convex continuous function
with arbitrary precision on any compact condition space and compact convex decision space.
To show the universal approximation theorem of PMA and PLSE,
the continuous functions replacing network parameters of MA and LSE
were constructed by continuous selection of multivalued subdifferential mappings.

The results of the numerical simulation \firstrev{show} that PLSE can approximate a target function
even in high-dimensional cases.
From the conditional optimization tests performed with unseen condition variables,
PLSE's minimizer and optimal value errors are small,
and the \firstrev{solving time} is scalable,
compared to existing approximators including
ordinary and convex universal approximators as well as
a practically suggested parameterized convex approximator.

The proposed parameterized convex universal approximators,
such as PLSE,
\firstrev{may be used as} useful approximators for
i) decision-making problems, e.g.,
continuous-action reinforcement learning,
and ii) applications to differentiable convex programming.
Future research directions may include
i) surrogate model approaches to cover nonparameterized-convex function approximation
by PMA and PLSE
and ii) \firstrev{various} applications to decision-making with dynamical systems including \firstrev{aerospace and robotic systems}.

\subsection{Acknowledgments}
This work was supported by the National Research Foundation of Korea (NRF) grant funded by the Korean government (MSIT) (No. 2019R1A2C2083946).

\begin{appendices}
\section{Proof of \autoref{thm:eps_selection_of_subdiff_mapping}}
\setcounter{equation}{0}  % reset equation numbering
\renewcommand{\theequation}{\thesection.\arabic{equation}}  % equations; (A.1), (B.2), etc.
\label{sec:proof_of_eps_selection_of_subdiff_mapping}
The proof of \autoref{thm:eps_selection_of_subdiff_mapping}
is followed by the following Lemmas.

\begin{lemma}
    \label{lemma:pointwise_supremum_is_lsc}
    The pointwise supremum of a collection of l.s.c. functions is l.s.c.
\end{lemma}
\begin{proof}
    Let $\mathcal{F}$ be a collection of l.s.c. functions on $X$ to $\mathbb{R}$.
    Let $g:X \to \mathbb{R}$ be a pointwise supremum function of $\mathcal{F}$, i.e.,
    $g(x) := \sup_{f \in \mathcal{F}} f(x)$.
    Given $x_{0} \in X$, for any $\epsilon > 0$,
    there exists an l.s.c. function $f' \in \mathcal{F}$
    such that $g(x_{0}) - \frac{\epsilon}{2} < f'(x_{0})$.
    Since $f'$ is l.s.c., there exists a neighborhood $U$ of $x_{0}$ such that
    \firstrev{$f'(x_{0}) - \frac{\epsilon}{2} \leq f'(x) \leq g(x), \forall x \in U$}.
    Hence, $g(x_{0}) - \epsilon \leq g(x), \forall x \in U$,
    which implies that $g$ is l.s.c.
\end{proof}

\begin{lemma}
    \label{lemma:a_mapping_is_lsc}
    \firstrev{
        Let $f:X \times U \to \mathbb{R}$ be a parameterized convex continuous function.
        Let $f_{x}(\cdot) := f(x, \cdot)$ for any $x \in X$.
        % For any $(x, u) \in X \times U$,
        % let $\partial f_{x}(u)
        % := \{ u^{*} \in \mathbb{R}^{m} \vert f_{x}(\tilde{u}) \geq f_{x}(u) + \langle u^{*}, \tilde{u} - u \rangle, \forall \tilde{u} \in U \}$ \firstrev{be the subdifferential of $f_{x}$ at $u$}.
        Then,
        % given $u \in U$,
        a mapping $(x, u^{*}) \mapsto f_{x}^{*}(u^{*}) := \sup_{u \in U} \{ \langle u, u^{*} \rangle -f_{x}(u) \}$ is l.s.c.
        % on $X \times \partial f_{x}(u)$.
        on $X \times \mathbb{R}^{m}$.
    }
\end{lemma}
\begin{proof}
    Fix $u \in U$.
    From the continuity of $f$, given $x \in X$,
    $\forall \epsilon > 0$, $\exists \delta_{x} > 0$ such that
    $\lVert f_{x}(u) - f_{x'}(u) \rVert < \frac{\epsilon}{2}, \forall x' \in X$
    where $\lVert x - x' \rVert < \delta_{x}$.
    Choose $\delta := \min \left\{ \frac{\epsilon}{2} \frac{1}{1+\lVert u \rVert}, \delta_{x} \right\}$.
    \firstrev{Given $(x, u^{*}) \in X \times \mathbb{R}^{m}$,}
    \begin{equation}
        \begin{split}
            & \left\lvert \left( \langle u, u^{*} \rangle - f_{x}(u) \right) - \left( \langle u, u^{*'} \rangle - f_{x'}(u) \right)\right\rvert
            \\
            & \leq \left\lvert \langle u, u^{*} - u^{*'} \rangle \right\rvert
            + \left\lvert f_{x}(u) - f_{x'}(u) \right\rvert
            \\
            & \leq \left\lVert u^{*} - u^{*'} \right\rVert \lVert u \rVert
            + \left\lvert f_{x}(u) - f_{x'}(u) \right\rvert
            \\
            & < \delta \lVert u \rVert + \frac{\epsilon}{2}
        \leq \frac{\epsilon}{2} \frac{\lVert u \rVert}{1 + \lVert u \rVert} + \frac{\epsilon}{2}
            < \epsilon,
        \end{split}
    \end{equation}
    for all $(x', u^{*'})$ such that \firstrev{$\lVert (x, u^{*}) - (x', u^{*'}) \rVert < \delta $,
    which implies $\lVert u^{*} - u^{*'} \rVert < \delta $}.
    That is,
    \firstrev{$(x, u^{*}) \mapsto \langle u, u^{*} \rangle - f_{x}(u)$ is continuous for any $u \in U$.}
    \textit{A fortiori}, it is l.s.c.
    By \autoref{lemma:pointwise_supremum_is_lsc},
    % given $u \in U$,
    the mapping \firstrev{$(x, u^{*}) \mapsto f_{x}^{*}(u^{*})$}
    is l.s.c
    \firstrev{on $X \times \mathbb{R}^{m}$.}
\end{proof}

\begin{lemma}
    \label{lemma:subdifferential_mapping_is_uhc}
    Suppose that all assumptions of
    \firstrev{\autoref{thm:eps_selection_of_subdiff_mapping}} hold.
    Then, given $u \in \accentset{\circ}{U}$, $\Gamma_{u}(x) := \partial f_{x}(u)$ is a nonempty-convex-valued
    u.h.c. multivalued function.
\end{lemma}
\begin{proof}
    First, we show that $\text{Graph}(\Gamma_{u})$ is closed for any $u \in \accentset{\circ}{U}$.
    Given $u \in \accentset{\circ}{U}$,
    it is sufficient to show that $\forall x_{i} \to x$,
    $u_{i}^{*} \to u^{*} \in \Gamma_{u}(x)$ as $i \to \infty$
    where $u_{i}^{*} \in \Gamma_{u}(x_{i})$.
    By \cite[Theorem 23.5]{rockafellarConvexAnalysis1970},
    the following relation holds,
    \begin{equation}
        \begin{split}
            & u_{i}^{*} \in \Gamma_{u}(x_{i}) = \partial f_{x_{i}}(u)
            \\ & \iff \langle u, u^{*} \rangle \geq f_{x_{i}}(u) + f_{x_{i}}^{*}(u_{i}^{*}).
        \end{split}
    \end{equation}
    By \autoref{lemma:a_mapping_is_lsc} and the continuity of $f$,
    taking $\liminf$ implies that
    \begin{equation}
        \begin{split}
            &\langle u, u^{*} \rangle
            \\
            &= \liminf_{i \to \infty} \langle u, u_{i}^{*} \rangle
            \geq \liminf_{i \to \infty}
            \left( f_{x_{i}}(u) + f_{x_{i}}^{*}(u_{i}^{*}) \right)
            \\
            &\geq f_{x}(u) + f_{x}^{*}(u^{*})
            \iff u^{*} \in \partial f_{x}(u) = \Gamma_{u}(x).
        \end{split}
    \end{equation}
    That is, $\text{Graph}(\Gamma_{u})$ is closed.
    Furthermore,
    given $u \in \accentset{\circ}{U}$,
    $\Gamma_{u}(x) = \partial f_{x}(u)$ is nonempty
    and closed convex by definition for all $x \in X$ \cite[Theorem 23.4]{rockafellarConvexAnalysis1970}.
    Since $f_{x}$ is $L$-Lipschitz for all $x \in X$,
    the codomain of $\Gamma_{u}$ is $B({0}, L)$ compact.
    By the converse statement of \cite[Proposition 1.4.8]{aubinSetValuedAnalysis2009},
    $\Gamma_{u}$ is u.h.c. for any $u \in \accentset{\circ}{U}$.
    Therefore, $\Gamma_{u}$ is a nonempty-convex-valued u.h.c. multivalued function.
\end{proof}

Now, we are ready to prove \autoref{thm:eps_selection_of_subdiff_mapping}.
\begin{proof}[Proof of \autoref{thm:eps_selection_of_subdiff_mapping}]
    Note that the following proof is based on
    \cite[Theorem 9.2.1]{aubinSetValuedAnalysis2009}
    and \cite[Theorem 2, Section 0]{aubinDifferentialInclusionsSetValued1984}.
    Fix $\epsilon > 0$.
    Given $u_{i} \in \accentset{\circ}{U}$,
    since $\Gamma_{u_{i}}$ is u.h.c. by \autoref{lemma:subdifferential_mapping_is_uhc},
    for every $x \in X$, there exists $\delta_{x} \in (0, 2 \epsilon)$
    such that
    $\Gamma_{u_{i}}(y) \subset \Gamma_{u_{i}}(x) + \frac{\epsilon}{2} B$, $\forall y \in B(x, \delta_{x})$.
    The collection of balls \firstrev{$\{ \accentset{\circ}{B}_{X}(x, \delta_{x}/4) \}_{x \in X}$} covers $X$.
    Since $X$ is compact,
    there exists a finite sequence $J$ of indices such that
    the collection of balls \firstrev{$\{ \accentset{\circ}{B}_{X}(x_{j}, \delta_{x_{j}}/4) \}_{j \in J}$} also covers $X$.
    We set $\delta_{j} = \delta_{x_{j}}$
    and take a locally Lipschitz partition of unity $\{ a_{j} \}_{j \in J}$
    subordinated to this covering \cite[Theorem 2, Section 0]{aubinDifferentialInclusionsSetValued1984}.
    That is, $a_{j}: X \to [0, 1]$ is a locally Lipschitz function
    such that $a_{j}$ vanishes outside of \firstrev{$B_{X}(x_{j}, \delta_{j}/4)$}, $\forall j \in J$,
    and $\sum_{j \in J}a_{j}(x) = 1$, $\forall x \in X$.
    Since $X$ is compact,
    a locally Lipschitz function $a_{j}$ is in fact Lipschitz on $X$.
    Note that $a_{j}, x_{j}, \delta_{j}$,
    and the Lipschitz constant of $a_{j}$ (namely, $L_{j}$) is independent of $i$.
    Let us associate with every $j \in J$ a point
    \firstrev{$y_{i, j} \in \Gamma_{u_{i}}(B_{X}(x_{j}, \delta_{j}/4))$}
    and define the map \firstrev{$\hat{u}_{\epsilon, i}^{*}: X \to \mathbb{R}^{m}$}
    as $\hat{u}_{\epsilon, i}^{*}(x) := \sum_{j \in J} a_{j}(x) y_{i, j}$,
    which is an $\epsilon$-selection of $\Gamma_{u_{i}}$,
    whose values are in the convex hull of the image of $\Gamma_{u_{i}}$ \cite[Theorem 9.2.1]{aubinSetValuedAnalysis2009}.
    Then, for all $x, x' \in X$,
    \firstrev{
        \begin{equation}
            \begin{split}
                \lVert \hat{u}_{\epsilon, i}^{*}(x)
                -\hat{u}_{\epsilon, i}^{*}(x') \rVert
            & = \left\lVert \sum_{j \in J} \left( a_{j}(x) - a_{j}(x') \right) y_{i, j} \right\rVert
            \\
            & \leq \sum_{j \in J} \left\lVert \left( a_{j}(x) - a_{j}(x') \right) y_{i, j} \right\rVert
            \\
            & \leq \sum_{j \in J} \lvert \left( a_{j}(x) - a_{j}(x') \right) \rvert \left\lVert y_{i, j} \right\rVert
            ,
            \end{split}
        \end{equation}
        and since $y_{i,j} \in \Gamma_{u_{i}}(B_{X}(x_{j}, \delta_{j} / 4)) \subset B(\{0\}, L)$ due to the Lipschitzness assumption,
        \begin{equation}
            \begin{split}
                \lVert \hat{u}_{\epsilon, i}^{*}(x)
                -\hat{u}_{\epsilon, i}^{*}(x') \rVert
            & \leq \sum_{j \in J} \lvert \left( a_{j}(x) - a_{j}(x') \right) \rvert \left\lVert y_{i, j} \right\rVert
            \\
            & \leq L \sum_{j \in J} \lvert \left( a_{j}(x) - a_{j}(x') \right) \rvert
            \\
            & \leq \tilde{L} \lVert \left( x - x' \right) \rVert
            ,
            \end{split}
        \end{equation}
    where
    \firstrev{$\tilde{L} = L \left( \sum_{j \in J} L_{j} \right)$},
    $\forall i \in \mathbb{N}$.
    }
    Hence, $\{ \hat{u}_{\epsilon, i}^{*} \}_{i \in \mathbb{N}}$
    is an equi-Lipschitz sequence of $\epsilon$-selections of
    $\{ \Gamma_{u_{i}} \}_{i \in \mathbb{N}}$.
\end{proof}

\section{Proof of \autoref{thm:pma_is_a_parameterized_convex_universal_approximator}}
\setcounter{equation}{0}  % reset equation numbering
\label{sec:proof_of_pma_is_a_parameterized_convex_universal_approximator}
Similar to
% \autoref{sec:proof_of_eps_selection_of_subdiff_mapping},
\hyperref[sec:proof_of_eps_selection_of_subdiff_mapping]{Appendix A},
the proof of \autoref{thm:pma_is_a_parameterized_convex_universal_approximator}
is followed by the following Lemmas.

\begin{lemma}
    \label{lemma:uniform_convergence_on_a_dense_set}
    Let $X$ be a metric space. Let $Y$ be a dense subset of $X$.
    Let $V$ be a Banach space.
    Let $\{ f_{n} \}_{n \in \mathbb{N}}$ be a sequence of continuous functions $f_{n} : X \to V$.
    If $f_{n} \to f$ uniformly on $Y$, then $f_{n} \to f$ uniformly on $X$ as well.
\end{lemma}
\begin{proof}
    From the uniform convergence of $\{ f_{n} \}_{n \in \mathbb{N}}$,
    for all $\epsilon > 0$, there exists $N \in \mathbb{N}$ such that
    $\lVert f_{n} - f \rVert_{\infty} < \frac{\epsilon}{2}$ on $Y$, $\forall n \geq N$.
    Then, for all $m, n \geq N$,
    the following inequality holds,
    \begin{equation}
        \lVert f_{n} - f_{m} \rVert_{\infty}
        \leq \lVert f_{n} - f \rVert_{\infty} + \lVert f - f_{m} \rVert_{\infty} < \frac{\epsilon}{2} + \frac{\epsilon}{2} = \epsilon.
    \end{equation}
    Given $x \in X$,
    there exists a sequence $\{ y_{n} \in Y \}_{n \in \mathbb{N}}$ such that $y_{n} \to x$ as $n \to \infty$ since $Y$ is dense in $X$.
    It can be deduced from the continuity of $f_{n}$ and $f_{m}$ that
    $\lVert f_{n}(x) - f_{m}(x) \rVert_{\infty} \leq \epsilon$.
    Since $V$ is complete, $\{ f_{n} \}_{n \in \mathbb{N}}$
    uniformly converges on $X$ and should be $f$.
\end{proof}

\begin{lemma}
    \label{lemma:equicontinuous_sequence}
    Given an equicontinuous sequence of continuous functions,
    $\{ g_{i} \}_{i \in \mathbb{N}}$,
    the sequence $\{ x \mapsto \sup_{1 \leq i \leq I} \{ g_{i}(x) \} \}_{I \in \mathbb{N}}$ is also equicontinuous.
\end{lemma}
\begin{proof}
    Given $I \in \mathbb{N}$,
    let $f_{I}(x) := \sup_{1 \leq i \leq I} g_{i}(x)$.
    Since $\{ g_{i} \}_{i \in \mathbb{N}}$ is equicontinuous,
    for all $\epsilon > 0$, there exists $\delta > 0$ such that
    \begin{equation}
        g_{i}(x)
        = g_{i}(x') + \left( g_{i}(x) - g_{i}(x') \right)
        < g_{i}(x') + \epsilon,
    \end{equation}
    for all $x, x'$ such that $\lVert x - x' \rVert < \delta$,
    $\forall i = 1, \ldots, I$.
    Taking the supremum yields
    \begin{equation}
        \sup_{1 \leq i \leq I} g_{i}(x) \leq \sup_{1 \leq j \leq I} g_{j}(x') + \epsilon,
    \end{equation}
    which implies $f_{I}(x) - f_{I}(x') \leq \epsilon$.
    Swapping $x$ and $x'$ also implies
    $\lvert f_{I}(x) - f_{I}(x') \rvert \leq \epsilon$.
    Therefore,
    the sequence $\{ f_{I} \}_{I \in \mathbb{N}}$ is equicontinuous.
\end{proof}

\begin{lemma}
    \label{lemma:moreau_yosida_regularisation}
    Given a parameterized convex continuous function $f: X \times U \to \mathbb{R}$,
    a Moreau-Yosida regularization \cite{strombergRegularizationBanachSpaces1996} of $f$,
    modified in this study for parameterized convex functions,
    is defined for some $\eta > 0$ as
    \begin{equation}
        \label{eq:moreau_yosida_regularisation}
            \tilde{f}_{\eta}(x, u)
            := \min_{u' \in U} \left\{
                \frac{1}{2\eta} \lVert u - u' \rVert^{2} + f(x, u')
            \right\},
    \end{equation}
    and the proximal operator is defined as
    \begin{equation}
        \textup{prox}_{f}(x, u)
        :=  \argmin_{u' \in U} \left(
            \frac{1}{2 \eta} \lVert u - u' \rVert^{2} + f(x, u')
        \right).
    \end{equation}
    For all $\eta' \geq \eta$,
    The following conditions hold:
    \begin{itemize}
        \item $\tilde{f}_{\eta} \to f$ pointwise as $\eta \to 0^{+}$.
        \item $\eta \mapsto \tilde{f}_{\eta}$ is monotonically decreasing and $\tilde{f}_{\eta} \leq f$ for all $\eta > 0$.
        \item $\tilde{f}_{\eta}$ is continuous for all $\eta > 0$.
    \end{itemize}
\end{lemma}
\begin{proof}
    For the first statement,
    readers are referred to \cite[Proposition 4.1.5]{hiriart-urrutyConvexAnalysisMinimization1993}.

    For the second statement,
    given $(x, u) \in X \times U$,
    the following inequality holds for any $\eta' \geq \eta$,
    \begin{equation}
        \begin{split}
            \tilde{f}_{\eta}(x, u)
            = & \frac{1}{2 \eta}\lVert u - u^{*} \rVert^{2}
            + f(x, u^{*})
            \\
            \geq & \frac{1}{2 \eta'}\lVert u - u^{*} \rVert^{2}
            + f(x, u^{*})
            \geq \tilde{f}_{\eta'}(x, u),
        \end{split}
    \end{equation}
    where $u^{*} \in \text{prox}_{f}(x, u)$,
    for all $(x, u) \in X \times U$.
    Additionally, it is trivial from the definition of the modified Moreau-Yosida regularization to find that $\tilde{f}_{\eta} \leq f$ for all $\eta > 0$.
    Thus concludes the proof of the second statement.

    For the third statement,
    consider a continuous function $h: X_{1} \times X_{2} \to \mathbb{R}$
    over compact subspaces $X_{1} \subset \mathbb{R}^{n}$ and $X_{2} \subset \mathbb{R}^{m}$.
    Then, for any $x_{1} \in X_{1}$,
    we can define the minimum function $m(x_{1}) := \min_{x_{2} \in X_{2}} h(x_{1}, x_{2})$.
    Since $h$ is continuous, it is uniformly continuous on $X_{1} \times X_{2}$ compact.
    Hence, $\forall \epsilon > 0$, $\exists \delta > 0$ such that
    $h(x_{1}, x_{2}) > h(x_{1}', x_{2}') - \epsilon$ for all $(x_{1}, x_{2}), (x_{1}', x_{2}') \in X_{1} \times X_{2}$
    where $\lVert (x_{1}, x_{2}) - (x_{1}', x_{2}') \rVert < \delta$.
    For fixed $\overline{x}_{1}$,
    there exists $\overline{x}_{2} \in X_{2}$ such that
    $m(\overline{x}_{1}) = h(\overline{x}_{1}, \overline{x}_{2})$.
    Then,
    \begin{equation}
        m(\overline{x}_{1})
        = h(\overline{x}_{1}, \overline{x}_{2})
        > h(x_{1}, \overline{x}_{2}) - \epsilon
        \geq m(x_{1}) - \epsilon,
    \end{equation}
    for all $x \in X_{1}$
    such that $\lVert \overline{x}_{1} - x_{1} \rVert < \delta$.
    By symmetry,
    we can show that
    $\lvert m(x_{1}) - m(\overline{x}_{1}) \rvert < \epsilon$.
    Hence, $m$ is continuous on $X_{1}$.
    Substituting $(h, x_{1}, x_{2}, m) \leftarrow (f, (x, u), u', \tilde{f}_{\eta})$ concludes the proof of the third statement.
\end{proof}

Then,
the proof of
\autoref{thm:pma_is_a_parameterized_convex_universal_approximator}
can be shown as follows.
\begin{proof}[Proof of \autoref{thm:pma_is_a_parameterized_convex_universal_approximator}]
    First, suppose that $U$ has a nonempty interior
    and for some $L > 0$, $f(x, \cdot)$ is $L$-Lipschitz for all $x \in X$.
    These assumptions will be relaxed at the end of the proof.

    One can choose a sequence $\{ u_{i} \in \accentset{\circ}{U} \}_{i \in \mathbb{N}}$
    such that the set
    $\tilde{U} := \{ u_{i} \in \accentset{\circ}{U} \vert i \in \mathbb{N}\}$
    of all points of the sequence is dense in $U$.
    As mentioned in \cite{calafioreLogSumExpNeuralNetworks2020},
    an example is a set of all points
    whose elements are rational numbers in $\accentset{\circ}{U}$.

    By Tychonoff's theorem \cite[Theorem 37.3]{munkresTopology2014},
    $X \times U$ is a compact subspace of $\mathbb{R}^{n} \times \mathbb{R}^{m}$.
    By \cite[Theorem 4.19]{rudinPrinciplesMathematicalAnalysis1976},
    continuity of $f$ on $X \times U$ implies that
    $f$ is uniformly continuous on $X \times U$.
    Hence,
    given $u \in U$,
    $\forall \epsilon_{1} > 0$, $\exists \delta_{1} > 0$
    such that $\lvert f(x, u) - f(x', u) \rvert < \epsilon_{1}$,
    $\forall x, x' \in X$ where $\lVert x - x' \rVert < \delta_{1}$.

    By \autoref{thm:eps_selection_of_subdiff_mapping},
    for any $\epsilon > 0$,
    there exists an equi-Lipschitz sequence $\{ \hat{u}_{\epsilon, i}^{*} \}$ of functions with Lipschitz constant of $L_{\epsilon}$
    such that $\hat{u}_{\epsilon, i}^{*}$ is an $\epsilon$-selection
    of $\Gamma_{u_{i}}: X \to \mathbb{R}^{m}$
    where $\Gamma_{u_{i}}(x) := \partial f_{x}(u_{i})$,
    $\forall x \in X$, that is,
    $\text{Graph}(\hat{u}_{\epsilon, i}^{*}) \subset B(\text{Graph}(\Gamma_{u_{i}}), \epsilon)$.
    Then, given $x \in X$, $i \in \mathbb{N}$,
    $\forall \delta_{2} > 0$,
    $\exists x_{i}' \in X$, $u_{\epsilon, i}^{*'} \in \partial f_{x_{i}'}(u_{i}) = \Gamma_{u_{i}}(x_{i}')$
    such that $\lVert (x, \hat{u}_{\epsilon, i}^{*}(x)) - (x_{i}', u_{i}^{*'}) \rVert < \delta_{2}$.
    For any $\epsilon > 0$,
    consider $f_{\epsilon}(x, u)
    := \sup_{i \in \mathbb{N}} \left\{
        f(x, u_{i}) + \langle \hat{u}_{\epsilon, i}^{*}(x), u - u_{i} \rangle
    \right\}$.
    Then, $\forall \epsilon_{2} > 0$,
    setting $\delta_{2} := \min \{ \delta_{1}, \epsilon_{2} / \text{diam}(U) \}$ implies that
    \begin{equation}
        \begin{split}
            & \Big \lvert
                \left(
                    f(x, u_{i})
                    + \langle
                        \hat{u}_{\delta_{2}, i}^{*}, u - u_{i}
                    \rangle
                \right)
                \\
            &- \left(
                    f(x_{i}', u_{i})
                    + \langle
                    \hat{u}_{\delta_{2}, i}^{*'}, u - u_{i}
                    \rangle
                \right)
            \Big \rvert
            \\
            & \leq \lvert f(x, u_{i}) - f(x_{i}', u_{i}) \rvert
            \\
            &+ \lVert \hat{u}_{\delta_{2}, i}^{*}(x) - u_{\delta_{2}, i}^{*'} \rVert \text{diam}(U)
            < \epsilon_{1} + \epsilon_{2},
        \end{split}
    \end{equation}
    for all $i \in \mathbb{N}$.
    Hence, given $u_{i} \in \tilde{U}$,
    \begin{equation}
        \begin{split}
            & f(x, u_{i})
            + \langle \hat{u}_{\delta_{2}, i}^{*}(x), u - u_{i} \rangle
            \\
            & < f(x_{i}', u_{i})
            + \langle \hat{u}_{\delta_{2}, i}^{*'}, u - u_{i} \rangle
            + \epsilon_{1} + \epsilon_{2}
            \\
            & \leq f(x_{i}', u) + \epsilon_{1} + \epsilon_{2}
            < f(x, u) + 2\epsilon_{1} + \epsilon_{2},
        \end{split}
    \end{equation}
    for all $(x, u) \in X \times U$ and $i \in \mathbb{N}$.
    Note for any $u \in \tilde{U}$ that $f_{\delta_{2}}(x, u) = f(x, u)$, $\forall x \in X$.
    Then, $f(x, u) \leq f_{\delta_{2}}(x, u) < f(x, u) + 2 \epsilon_{1} + \epsilon_{2}$, $\forall (x, u) \in X \times \tilde{U}$,
    and it can be deduced that $f_{\delta_{2}}$ converges uniformly to $f$ on $X \times \tilde{U}$ as $\delta_{2} \to 0^{+}$.
    By \autoref{lemma:uniform_convergence_on_a_dense_set},
    $f_{\delta_{2}}$ converges uniformly to $f$
    on $\overline{X \times \tilde{U}}
    = \overline{X} \times \overline{\tilde{U}} = X \times U$
    as $\delta_{2} \to 0^{+}$.

    Next, we show that a sequence of functions,
    $\{ (x, u) \mapsto f(x, u_{i}) + \langle \hat{u}_{\epsilon, i}^{*},
        u - u_{i} \rangle \}$ is equicontinuous.
        From the uniform continuity of $f$ on $X \times U$,
        $\forall \epsilon_{3} > 0$, $\exists \delta_{3}' > 0$
        such that $\lvert f(x, u) - f(x', u') \rvert < \frac{\epsilon_{3}}{3}$, $\forall (x, u), (x', u') \in X \times U$ where $\lVert (x, u) - (x', u') \rVert < \delta_{3}'$.
        Letting $\delta_{3} := \min \{
            \delta_{3}',
            \frac{\epsilon_{3}}{3 L_{\epsilon} \text{diam}(U)},
            \frac{\epsilon_{3}}{3(L + \epsilon) \text{diam}(U)}
        \}$ implies
        \begin{equation}
            \begin{split}
                & {\big \lvert}
                    \left(
                        f(x, u_{i}) + \langle \hat{u}_{\epsilon, i}^{*} (x), u - u_{i} \rangle
                    \right)
                    \\
                & - \left(
                        f(x', u_{i}) + \langle \hat{u}_{\epsilon, i}^{*} (x'), u' - u_{i}\rangle
                    \right)
                    {\big \lvert}
                \\
                & \leq \lvert f(x, u_{i}) - f(x', u_{i}) \rvert
                \\
                &+ \lvert \langle \hat{u}_{\epsilon, i}^{*}(x) - \hat{u}_{\epsilon, i}^{*} (x'),
                u - u_{i}
                \rangle \rvert
                + \lvert \langle
                \hat{u}_{\epsilon, i}^{*}(x'), u' - u
                \rangle
                \rvert
                \\
                & \leq \lvert f(x, u_{i}) - f(x', u_{i}) \rvert
                \\
                &+ \left(
                    \lVert \hat{u}_{\epsilon, i}^{*}(x) - \hat{u}_{\epsilon, i}^{*}(x') \rVert
                    + \lVert \hat{u}_{\epsilon, i}^{*}(x') \rVert
                \right) \text{diam}(U)
                \\
                & < \frac{\epsilon_{3}}{3}
                + \left(
                    L_{\epsilon}\lVert x - x' \rVert + (L + \epsilon)
                \right) \text{diam}(U)
                \\
                & \leq \frac{\epsilon_{3}}{3}
                + \frac{\epsilon_{3}}{3} + \frac{\epsilon_{3}}{3}
                = \epsilon_{3},
            \end{split}
        \end{equation}
        for all $(x, u), (x', u') \in X \times U$ such that
        $\lVert (x, u) - (x', u') \rVert < \delta_{3}$.
        \firstrev{
            Note for all $i \in \mathbb{N}$ that $\lVert \hat{u}_{\epsilon, i}(x') \rVert$ is equal to or less than
            $L+\epsilon$ since $\hat{u}_{\epsilon, i}$ is an $\epsilon$-selection of $\Gamma_{u_{i}} \subset B(\{0\}, L)$
            from the Lipschitzness assumption.
        }
        By \autoref{lemma:equicontinuous_sequence},
        given $\epsilon > 0$, $\{ f_{\epsilon, I} \}_{I \in \mathbb{N}}$
        is an equicontinuous sequence of functions where
        \begin{equation}
            \begin{split}
                f_{\epsilon, I}(x, u)
                & := \sup_{1 \leq i \leq I} \{
                    f(x, u_{i}) +
                    \langle \hat{u}_{\epsilon, i}^{*},
                    u - u_{i}
                    \rangle
                \\
                & = \max_{1 \leq i \leq I} \{
                    f(x, u_{i}) +
                    \langle \hat{u}_{\epsilon, i}^{*},
                    u - u_{i}
                    \rangle
                \}.
                \end{split}
        \end{equation}
        $f_{\epsilon, I}$ obviously converges to $f_{\epsilon}$
        pointwise on $X \times U$ as $I \to \infty$,
        and $\{ f_{\epsilon, I} \}_{I \in \mathbb{N}}$
        is equicontinuous.
        Then, $f_{\epsilon, I}$ converges to $f_{\epsilon}$
        uniformly on $X \times U$ by the Arzela\`-Ascoli theorem \cite[Theorem 7.25]{rudinPrinciplesMathematicalAnalysis1976}.
        From the above statements,
        we can conclude the proof with assumptions of
        i) the Lipschitzness of $f(x, \cdot)$ for all $x \in X$
        and ii) the nonempty interior of $U$:
        $\forall \epsilon > 0$, $\exists \delta > 0$ such that
        $\lVert f - f_{\delta} \rVert_{\infty} < \frac{\epsilon}{4}$.
        Additionally, $\forall \delta > 0$, $\exists I \in \mathbb{N}$ such that
        $\lVert f_{\delta} - f_{\delta, I} \rVert_{\infty} < \frac{\epsilon}{4}$.
        Therefore, 
        $\forall \epsilon > 0$, $\exists \delta > 0$, $I \in \mathbb{N}$
        such that
        $
        \lVert f - f_{\delta, I} \rVert_{\infty}
        \leq \lVert f - f_{\delta} \rVert_{\infty}
        + \lVert f_{\delta} - f_{\delta, I} \rVert_{\infty}
        < \frac{\epsilon}{4} + \frac{\epsilon}{4} = \frac{\epsilon}{2}.
        $

        \firstrev{Now, let us} relax the Lipschitzness assumption by Moreau-Yosida regularization.
        Given the parameterized convex continuous function $f$,
        consider regularized functions $\tilde{f}_{\eta}$
        for any $\eta > 0$, defined in Eq. \eqref{eq:moreau_yosida_regularisation}.
        By \autoref{lemma:moreau_yosida_regularisation},
        $\{ \tilde{f}_{\eta} \}_{\eta > 0}$ is a collection of monotonically decreasing continuous functions (with respect to $\eta$) on a compact space $X \times U$,
        which converges to a continuous function $f$ pointwise.
        By Dini's theorem \cite[Theorem 7.13]{rudinPrinciplesMathematicalAnalysis1976},
        the collection $\{ \tilde{f}_{\eta} \}_{\eta > 0}$ converges to $f$ uniformly as well.
        That is, $\forall \epsilon > 0$, $\exists \eta > 0$ such that
        $\lVert f - \tilde{f}_{\eta} \rVert_{\infty} < \frac{\epsilon}{2}$.
        Moreover,
        from \cite[Example 3.4.4]{hiriart-urrutyConvexAnalysisMinimization1993},
        $\nabla_{u} \tilde{f}_{\eta}(x, u) = \frac{1}{\eta} (u - u^{*})$
        \firstrev{implies}
        $\lVert \nabla_{u} \tilde{f}_{\eta}(x, u) \rVert \leq \frac{1}{\eta} \text{diam}(U)$,
        where $u^{*} \in \text{prox}_{f}(x, u)$,
        for all $(x, u) \in X \times U$.
        Hence, $\tilde{f}(x, u)$ is $L_{\eta}$-Lipschitz for any $x \in X$ where $L_{\eta} := \frac{1}{\eta} \text{diam}(U)$.
        Combining the above statements achieves the relaxation of the Lipschitzness assumption,
        i.e.,
        $\lVert f - f_{\delta, I} \rVert_{\infty}
        \leq \lVert f - \tilde{f}_{\eta} \rVert_{\infty}
        + \lVert \tilde{f}_{\eta} - f_{\delta, I} \rVert_{\infty}
        < \frac{\epsilon}{2} + \frac{\epsilon}{2}
        = \epsilon$.
        Note that $(\tilde{f}_{\eta})_{\delta, I}$ is denoted by $f_{\delta, I}$ for notational brevity.

    Relaxation of the nonemptyness of the interior $\accentset{\circ}{U}$ of $U$ can be performed by expanding $U$, as described at the end of the proof of \cite[Theorem 2]{calafioreLogSumExpNeuralNetworks2020} and is omitted here for brevity.

        Rearranging $f_{\delta, I}$ as
        \begin{equation}
            f_{\delta, I}(x, u)
            = \max_{1 \leq i \leq I} \{
                \langle a_{i}(x), u \rangle + b_{i}(x)
            \},
        \end{equation}
        where $a_{i}: X \to \mathbb{R}^{m}$ and $b_{i}:X \to \mathbb{R}$ are continuous functions such that $a_{i}(x) = \hat{u}_{\epsilon, i}^{*}(x)$ and $b_{i}(x) = f(x, u_{i}) - \langle \hat{u}_{\epsilon, i}^{*}, u_{i} \rangle$, respectively,
        implies $f_{\delta, I} \in \mathcal{F}^{\text{PMA}}$, which concludes the proof.
\end{proof}

\end{appendices}

\bibliography{ref.bib}
\bibliographystyle{IEEEtran}


% \bibitem{ams}
% {\it{Mathematics into Type}}, American Mathematical Society. Online available: 

% \bibitem{oxford}
% T.W. Chaundy, P.R. Barrett and C. Batey, {\it{The Printing of Mathematics}}, Oxford University Press. London, 1954.

% \bibitem{lacomp}{\it{The \LaTeX Companion}}, by F. Mittelbach and M. Goossens

% \bibitem{mmt}{\it{More Math into LaTeX}}, by G. Gr\"atzer

% \bibitem{amstyle}{\it{AMS-StyleGuide-online.pdf,}} published by the American Mathematical Society

% \bibitem{Sira3}
% H. Sira-Ramirez. ``On the sliding mode control of nonlinear systems,'' \textit{Systems \& Control Letters}, vol. 19, pp. 303--312, 1992.

% \bibitem{Levant}
% A. Levant. ``Exact differentiation of signals with unbounded higher derivatives,''  in \textit{Proceedings of the 45th IEEE Conference on Decision and Control}, San Diego, California, USA, pp. 5585--5590, 2006.

% \bibitem{Cedric}
% M. Fliess, C. Join, and H. Sira-Ramirez. ``Non-linear estimation is easy,'' \textit{International Journal of Modelling, Identification and Control}, vol. 4, no. 1, pp. 12--27, 2008.

% \bibitem{Ortega}
% R. Ortega, A. Astolfi, G. Bastin, and H. Rodriguez. ``Stabilization of food-chain systems using a port-controlled Hamiltonian description,'' in \textit{Proceedings of the American Control Conference}, Chicago, Illinois, USA, pp. 2245--2249, 2000.

% \end{thebibliography}

% Biography text here without a photo.

\begin{IEEEbiography}[{\includegraphics[width=1in,height=1.25in,clip,keepaspectratio]{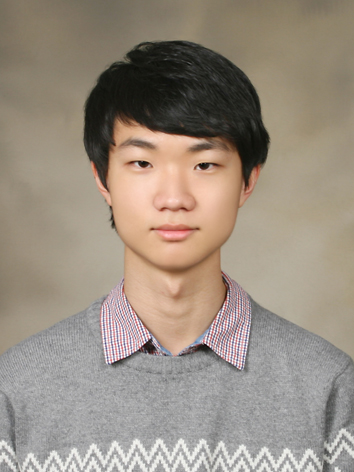}}]{Jinrae Kim}
received a B.S. degree in mechanical and aerospace engineering from Seoul National University, Republic of Korea, in 2017, where he is currently pursuing a Ph.D. degree in aerospace engineering. His current research interests include approximation, optimisation, and data-driven control.
\end{IEEEbiography}

\begin{IEEEbiography}[{\includegraphics[width=1in,height=1.25in,clip,keepaspectratio]{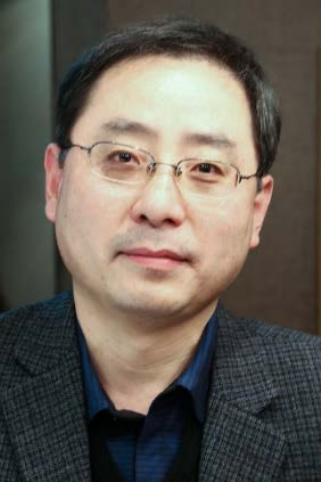}}]{Youdan Kim}
received B.S. and M.S. degrees in aeronautical engineering from Seoul National University, Republic of Korea, in 1983 and 1985, respectively, and the Ph.D. degree in aerospace engineering from Texas A\&M University in 1990. He joined the faculty of Seoul National University in 1992, where he is currently a Professor with the Department of Aerospace Engineering. His current research interests include aircraft control system design, reconfigurable control system design, path planning, and guidance techniques for aerospace systems.
\end{IEEEbiography}

\end{document}